\documentclass{article}

\usepackage[preprint]{neurips_2026}

\usepackage[utf8]{inputenc}
\usepackage[T1]{fontenc}
\usepackage{hyperref}
\usepackage{url}
\usepackage{microtype}
\usepackage{graphicx}
\IfFileExists{figures/tv_joint_cot_vs_noncot_grid.png}{\graphicspath{{figures/}}}{\graphicspath{{paper_results_full_repro_20260427_seed1234/figures/}}}
\usepackage{subcaption}
\usepackage{booktabs}
\usepackage{placeins}
\usepackage{float}
\usepackage{amsmath}
\usepackage{amssymb}
\usepackage{amsfonts}
\usepackage{mathtools}
\usepackage{amsthm}
\usepackage{xcolor}
\usepackage{tcolorbox}
\tcbuselibrary{breakable}
\usepackage{soul}
\usepackage[capitalize,noabbrev]{cleveref}

\tcbset{
  myexamplebox/.style={
    colback=blue!10!purple!10!white,
    colframe=white,
    boxrule=0pt,
    fontupper=\small,
    sharp corners,
    left=0.45em, right=0.45em, top=0.35em, bottom=0.35em,
    before skip=0.45em,
    after skip=0.45em,
    breakable,
    parbox=false,
  }
}

\theoremstyle{plain}
\newtheorem{theorem}{Theorem}[section]
\newtheorem{claim}{Claim}[section]

\newtheorem{corollary}[theorem]{Corollary}
\theoremstyle{definition}
\newtheorem{definition}[theorem]{Definition}
\newtheorem{assumption}[theorem]{Assumption}
\theoremstyle{remark}
\newtheorem{remark}[theorem]{Remark}

\title{Measuring Behavior Portability in Large Language Models}

\author{%
  Tianjia Dong \\
  Knowledge Lab \\
  University of Chicago \\
  \And
  Nadav Kunievsky \\
  Knowledge Lab \\
  University of Chicago \\
  \And
  James A. Evans \\
  Knowledge Lab \\
  Data Science Institute \\
  Department of Sociology \\
  University of Chicago
}

\begin{document}

\maketitle

\begin{abstract}
Large language models are increasingly deployed as autonomous decision makers, yet the behavioral mapping they exhibit can vary substantially across decision environments that are payoff-equivalent by construction—environments that share identical payoff-relevant structure but differ in surface presentation. This sensitivity renders suite-based evaluation fragile and raises a fundamental question of behavioral portability: how well does a behavioral mapping learned in one decision environment informative on another that preserves the same underlying incentive structure? We introduce a formal framework to measure this property. Our protocol fits an interpretable behavioral model on data pooled from a set of source environments and evaluates its out-of-sample predictive performance in a held-out target environment, benchmarking against an oracle trained directly on target data. Portability is quantified via a loss-agnostic measure that delivers worst-case bounds on the performance of the induced prediction-action mapping in the target environment. In controlled experiments spanning seven canonical economic decision problems, we document substantial and systematic portability losses, suggesting that behavioral characterizations of LLMs obtained in one decision environment cannot be assumed to transfer reliably to structurally equivalent alternatives.
\end{abstract}
\section{Introduction}

Large language models (LLMs) are increasingly deployed as \emph{decision makers}: given an input, they induce a distribution over actions---a recommendation, a choice, a tool call, or a prediction. Their appeal as delegates lies in flexibility: a decision problem typically interleaves features that should drive the decision with surrounding context that should not. For instance, an LLM tasked with evaluating loan applications receives candidate descriptions that combine payoff-relevant risk parameters---income, credit history, debt-to-income ratio---with payoff-irrelevant content such as biographical detail or narrative tone. That flexibility, however, makes evaluation brittle: an evaluator can probe only finitely many decision environments, raising a basic question: \emph{how portable is an LLM's behavior across decision environments that share the same payoff-relevant structure?} 

This question is consequential for deployment: what matters is not only how the model behaves on the environments we test, but whether the behavioral pattern recovered from those environments continues to hold under new payoff-equivalent descriptions. We refer to this out-of-environment stability as \emph{behavioral portability}: the extent to which a model's mapping from payoff-relevant features to actions remains stable across environments that preserve the same payoff structure but differ in payoff-irrelevant framing.

Most evaluation proceeds via \emph{prompt suites}: curated sets of environments meant to approximate deployment conditions. Performance on the suite is then extrapolated to environments not in it. That extrapolation implicitly assumes portability---once the task is fixed, the model's induced action distribution should be stable under payoff-irrelevant changes in framing. Yet even small reframings can shift LLM behavior, so portability is not something a finite suite can certify; it is an empirical object to be measured.

This issue is central to the adoption of AI systems in real-world decision environments such as labor markets, hiring, task allocation, and workplace evaluation, where users cannot fully specify every contingency in advance. As suggested by \citet{hadfieldmenell2018incomplete}, decision problems delegated to algorithms such as LLMs resemble incomplete contracts: principals cannot enumerate and validate the agent's behavior across the full scope of cases it will encounter, so successful delegation depends on the agent filling these gaps in ways consistent with the principal's intended objective. If an LLM's behavior depends on payoff-irrelevant features of how a problem is described, finite evaluations become an unreliable basis for deployment in open-ended environments.

We make this precise with a simple decomposition. A decision environment $e=(x,z)$ separates into the payoff-relevant features $x$, which enter the evaluator's utility, and the remaining features $z$ of how the problem is presented. Under a natural uniqueness condition on the evaluator's optimum, an aligned policy must depend on $e$ only through $x$: at any two environments with the same payoff-relevant component, it induces the same action distribution. Any aligned algorithm, therefore, responds only to $x$: payoff-equivalent reframings leave its behavior unchanged. This yields an operational test. If the policy responds only to $x$, the input--action mapping recovered in one environment must coincide with the mapping that holds in any payoff-equivalent environment; equivalently, a behavioral representation fit on a source environment must predict behavior in a held-out target environment as well as one fit directly on the target. Any predictive gap is therefore evidence that behavior depends on $z$ within the class, and hence that the policy cannot be aligned throughout it.

Concretely, for each task we construct many payoff-equivalent environments that vary only in $z$, and sample repeated model actions across many draws of payoff parameters $x$. On a subset of \emph{source} environments we fit an interpretable behavioral representation mapping $x$ to actions, then evaluate it out of sample in a held-out \emph{target} environment. As a predictive-transfer diagnostic, we compare the source-trained representation's target loss to the loss of the same representation fit directly on target data. Because we do not want our main measure to privilege a particular scoring rule, we also define a loss-agnostic discrepancy: the total variation distance, computed in the target environment, between two joint distributions over predicted and realized actions. In one distribution, the prediction comes from the representation fit on source environments and evaluated in the target; in the other, it comes from the target-trained benchmark. This gives a worst-case bound on portability loss: for any bounded criterion defined on the predicted--realized action pair, the source-trained and target-trained representations can differ in expected criterion value by at most this TV distance.


We apply the framework in a controlled experiment from experimental economics, where payoff structure is transparent. We study seven one-shot tasks (Dictator, Ultimatum, Trust, Public Goods, Beauty Contest, Lottery Choice, and Normal Form), and for each task construct a large set of decision environments that preserve the payoff mapping while varying only in framing and style. We evaluate several LLMs---GPT-4.1-nano \citep{openai2025gpt41}, Gemma-3-12B \citep{gemmateam2025gemma3}, Llama-3.1-8B and Llama-3.1-70B \citep{grattafiori2024llama3}, and DeepSeek-R1 \citep{guo2025deepseekr1}---under answer-only and chain-of-thought (CoT) prompting \citep{wei2022cot}. Game-theoretic tasks are a natural testbed for evaluating portability in LLMs, both because their payoff-relevant values are transparent and because they have increasingly been used to measure LLM behavioral properties---including fairness, cooperation, trust, risk preferences, strategic reasoning, and rationality \citep{horton2023homo,guo2023gpt,brookins2024playing,chen2023economicrationality,mei2024turing,lore2024strategic,xie2024trust,akata2025repeated}.

Our results show that the tested LLMs do not demonstrate portability: behavioral mappings from payoff-relevant variables to actions that are learned in one environment often predict worse in another. CoT changes the portability levels of the model, on average improving portability, but not in a uniform way---in some cases portability improves, in others it does not. Finally, we find that reasoning models such as DeepSeek-R1 do better on portability across the tested tasks.

\paragraph{Related Literature}
Our portability question is a decision-theoretic analogue of dataset shift and domain adaptation, where behavior learned under one input distribution is deployed under another \citep{quinonero2009datasetshift,bendavid2010theory}. It also connects to causal transportability and invariant-prediction frameworks that formalize when knowledge transfers across environments \citep{bareinboim2012transportability,peters2016invariant,arjovsky2019irm}. In the LLM literature, robustness to prompt variation is studied through benchmark suites, prompt-sensitivity measures, multi-prompt evaluation, and behavioral stress tests \citep{liang2022helm,chatterjee-etal-2024-prompt,zhuo-etal-2024-prosa,chatterjee2024posix,mizrahi-etal-2024-state,cao2024worstprompt,zhu2024promptbench,zhu2023promptbench,mei2023assert,zheng2025promptanatomy,ribeiro2020checklist,goel2021robustnessgym,morris2020textattack}. Our contribution is to treat prompt sensitivity as \emph{behavioral portability}: the out-of-environment performance of an interpretable representation transported across payoff-equivalent textual frames, complementing work that uses LLMs as simulated economic agents \citep{aher2023using,horton2023homo,xie2025llmstrategic} and economic work on transfer performance \citep{andrews2022transfer}.

\section{Conceptual Framework and Motivation}\label{sec:conceptual}

We consider a Decision-Making Algorithm (DMA) that maps a decision environment
\(e\in\mathcal E\) into a distribution over a finite set of actions \(\mathcal A\). For LLMs, which are our main focus, an environment is a natural-language description of a decision problem, i.e., a prompt. In other
applications, \(e\) could collect all inputs available to a robot or software
system before making a decision.

Formally, let \(\Delta(\mathcal A)\) denote the simplex of probability
distributions over \(\mathcal A\). We represent the DMA by a policy rule
\[
\pi:\mathcal E\longrightarrow \Delta(\mathcal A).
\]
Next, we consider a human evaluator, with preferences over action--environment pairs, represented by a utility function
\[
u:\mathcal A\times\mathcal E\to\mathbb R.
\]
The key insight is that not every feature of the environment matters for the evaluator’s payoff. Some features define the decision problem itself and are therefore payoff-relevant, while others are not changing the underlying incentives structure. We formalize this distinction in the following assumption.

\begin{assumption}[Payoff-relevant features]\label{ass:payoff_relevant}
There exist sets \(\mathcal X\) and \(\mathcal Z\), maps
\(X:\mathcal E\to\mathcal X\) and \(Z:\mathcal E\to\mathcal Z\), and a function
\(\tilde u:\mathcal A\times\mathcal X\to\mathbb R\) such that each environment
\(e\in\mathcal E\) can be written as \(e=(x,z)\), where \(x=X(e)\) and
\(z=Z(e)\). The component \(x\) collects all payoff-relevant features of the
environment, while \(z\) summarizes all remaining payoff-irrelevant features, i.e.
Payoffs depend on the environment only through \(x\):
\[
u(a,e)=\tilde u\bigl(a,X(e)\bigr)
\qquad\text{for all }(a,e)\in\mathcal A\times\mathcal E.
\]
\end{assumption}

For example, in an ultimatum game, \(X(e)\) includes the stake, feasible offers, player roles, and the payoff consequences of acceptance and rejection. By contrast, \(Z(e)\) includes labels that do not change the payoff problem: whether the proposer is called ``Participant 365,'' ``Employee 47,'' or
``Player A.''

For each payoff-relevant state \(x\in\mathcal X\), let
\(\mathcal F(x)\subseteq\Delta(\mathcal A)\) denote the set of feasible distributions over actions. Given \(F\in\mathcal F(x)\), define the evaluator's von Neumann--Morgenstern expected utility by $U(F,x):=\mathbb E_{a\sim F}\bigl[\tilde u(a,x)\bigr]$. We make the following assumption.

\begin{assumption}[Uniqueness of the evaluator's optimum]\label{ass:unique}
For every \(x\in\mathcal X\), the set
\[
\arg\max_{F\in\mathcal F(x)} U(F,x)
\]
is a singleton. We denote its unique element by $B(x)$.
\end{assumption}

Notice that the uniqueness assumption is on the optimal feasible \emph{distribution} over actions, and not on the realized action drawn from that distribution. Next, we say that the algorithm is \emph{pointwise aligned} with the evaluator at environment \(e\) if it chooses a utility-maximizing distribution at the corresponding \(x=X(e)\).

\begin{definition}[Pointwise alignment]\label{def:alignment}
Given \(u\) satisfying Assumption~\ref{ass:payoff_relevant}, a policy \(\pi\) is
\emph{aligned} with the evaluator at environment \(e\in\mathcal E\) if
\[
\pi(e)\in\arg\max_{F\in\mathcal F(X(e))}U(F,X(e)).
\]
We say \(\pi\) is aligned on a set
\(\mathcal E'\subseteq\mathcal E\) if it is aligned at every
\(e\in\mathcal E'\).
\end{definition}

A DMA can respond to any aspect of the decision environment, including features that are irrelevant to the evaluator. Because the space of environmental variables may be large relative to the payoff-relevant components, it is useful to isolate when the DMA's behavior depends only on the payoff-relevant structure of the problem, which is potentially much smaller. As different applications may focus on different aspects of the induced action distribution, we define \emph{behavioral portability} relative to a measurable behavioral summary $T$, and require that this summary depend on the environment only through the payoff-relevant features.

\begin{definition}[\(T\)-behavioral portability]\label{def:portability}
Fix a set of environments \(\mathcal E'\subseteq\mathcal E\) and a measurable behavioral summary
\[
T:\Delta(\mathcal A)\times\mathcal X\to\mathcal Y.
\]
A policy \(\pi:\mathcal E\to\Delta(\mathcal A)\) is \(T\)-portable on
\(\mathcal E'\) if there exists a measurable map
\[
r_x:\mathcal X\to\mathcal Y
\]
such that, for every \(e\in\mathcal E'\),
\[
T(\pi(e),X(e))=r_x(X(e)).
\]
Equivalently, the behavioral summary \(T(\pi(e),X(e))\) depends on the
environment only through the payoff-relevant component \(X(e)\), and not through
the payoff-irrelevant component \(Z(e)\).
\end{definition}


\begin{tcolorbox}[myexamplebox]
\begin{remark}
Behavioral portability is related to prompt robustness, but it is not identical to it. Prompt robustness asks whether behavior is stable under changes in wording. Portability asks whether behavioral variation across environments is fully mediated by the payoff-relevant variables \(X(e)\). Thus, portability is a structured form of robustness: behavior may vary with the decision problem, but not with payoff-irrelevant features of the decision environment.    
\end{remark}
\end{tcolorbox}

The next claim gives the link from alignment to portability. Under uniqueness, pointwise alignment forces the full action distribution to reduce to a function of the payoff-relevant variables. Since the full distribution is itself one possible behavioral summary, this implies distribution-level portability; and because every other behavioral summary is a function of that distribution, it also implies \(T\)-portability for every measurable \(T\).

\begin{claim}[Alignment implies behavioral portability]\label{cl:invariance}
Suppose Assumptions~\ref{ass:payoff_relevant} and~\ref{ass:unique} hold, and that \(\pi\) is aligned with the evaluator on \(\mathcal E'\subseteq\mathcal E\). Then there exists a reduced policy $\pi_x:\mathcal X\to\Delta(\mathcal A)$, $\pi_x(x):=B(x)$, such that
\[
\pi(e)=\pi_x(X(e))
\qquad\text{for all }e\in\mathcal E'.
\]
and \(\pi\) is \(T\)-portable on \(\mathcal E'\) for every measurable behavioral summary $T$. In particular, for any \(e_1,e_2\in\mathcal E'\) such that $X(e_1)=X(e_2)=x$, we have
\[
\pi(e_1)=\pi(e_2)=B(x).
\]
\end{claim}

\emph{Proof.} See Appendix~\ref{app:proofs}.

Claim~\ref{cl:invariance} shows that, for the algorithm and evaluator to be aligned according to definition \ref{def:alignment}, the DMA's behavior must be organized around the same payoff-relevant variables \(X(e)\) that enter the evaluator's utility. If the full action distribution varies with payoff-irrelevant features \(Z(e)\), then there exist payoff-equivalent environments at which the algorithm implements different distributions over actions. Under uniqueness, at least one of these distributions must be strictly suboptimal for the evaluator.

The claim establishes portability as an implication of alignment. The next corollary states what portability buys for evaluation of fitted behavioral objects: if a policy is \(T\)-portable, then any loss comparing a fixed \(X\)-only representation with the behavioral summary \(T\) has the same expectation across environment distributions with the same \(X\)-marginal.

\begin{corollary}[Transportability of source-fitted \(X\)-only losses]
\label{cor:transport}
Fix \(\mathcal E'\subseteq\mathcal E\), and suppose that \(\pi\) is
\(T\)-portable on \(\mathcal E'\). Let
\(r_x:\mathcal X\to\mathcal Y\) denote a reduced summary satisfying
\[
T(\pi(e),X(e))=r_x(X(e))
\qquad\text{for all }e\in\mathcal E'.
\]
Let \(\mu_s\) and \(\mu_t\) be probability measures on
\(\mathcal E'\) with the same marginal distribution of \(X(e)\). Let \(\widehat{\mathcal Y}\) be a prediction space, and let
\[
\widehat h_s:\mathcal X\to\widehat{\mathcal Y}
\]
denote a measurable \(X\)-only representation fitted under the source
distribution \(\mu_s\). Then for any measurable loss or evaluation
criterion
\[
\ell:\widehat{\mathcal Y}\times\mathcal Y\times\mathcal X\to\mathbb R,
\]
with finite expectations,
\[
\mathbb E_{e\sim\mu_s}
\left[
\ell\left(
\widehat h_s(X(e)),
T(\pi(e),X(e)),
X(e)
\right)
\right]
=
\mathbb E_{e\sim\mu_t}
\left[
\ell\left(
\widehat h_s(X(e)),
T(\pi(e),X(e)),
X(e)
\right)
\right].
\]
\end{corollary}

\emph{Proof.} See Appendix~\ref{app:proofs}.

Combining Claim~\ref{cl:invariance} with Corollary~\ref{cor:transport} yields a simple evaluation implication of pointwise alignment. When the algorithm is aligned, every measurable behavioral summary \(T\) is portable, so once the policy function \(\widehat h\) has been learned, the expected loss from comparing \(\widehat h(X(e))\) to the portable behavioral summary is the same across environment distributions that share an \(X\)-marginal.

Corollary~\ref{cor:transport} captures how adoption decisions are made, where an evaluator observes the algorithm only on a finite evaluation suite and needs to decide on adoption. An evalautor, given her chosen behavioral summary \(T\),  learns a mapping \(\widehat h\) from \(X\) to $T(X)$ from the observed test behavior, and then evaluates it with a criterion \(\ell\) that measures whether the algorithm behaves appropriately given the payoff-relevant structure \(X(e)\). For example, the criterion may be predictive, such as loss against realized actions, or decision-oriented, such as the value of the induced action under a downstream objective.

The evaluator wants this finite-suite assessment to remain informative beyond the particular environments tested: a small loss on the suite should mean that the same \(X\)-based behavioral description continues to hold in other payoff-equivalent environments. Corollary~\ref{cor:transport} shows that this extrapolation is valid once the algorithm's behavior is portable: behavior learned in the training environment then generates the same loss in any payoff-equivalent target environment.

The corollary also gives a natural empirical diagnostic\footnote{We discuss explicitly how Corollary \ref{cor:transport} maps to our empirical measure in Appendix \ref{app:projection}}, which we implement in the next section. Specifically, we learn the mapping between the payoff-relevant variables and \(T\), then ask whether it continues to describe behavior in a held-out target environment with the same distribution of payoff-relevant variables. If behavior is portable, every candidate \(X\)-only representation has the same population risk across payoff-equivalent frames, so under a uniqueness condition the source-trained and target-trained population projections coincide. A gap between these two objects means that the evaluation is frame-specific: a criterion that appears small in one test suite need not remain small in a payoff-equivalent deployment frame. Such a gap shows that behavior depends on \(Z(e)\) in a way captured by the chosen summary and representation class, and is therefore inconsistent with alignment across the compared environments.

\begin{tcolorbox}[myexamplebox]
\begin{remark}
Portability is related to generalization, where the question is whether performance carries over to cases where the underlying distribution of inputs and and targets change. In standard generalziation questions the analyst fixes the input variables and the target label in advance and asks whether performance carries over to fresh draws of those same variables. The analyst controls the representation and chooses which variables enter the model, so payoff-irrelevant features can be excluded by construction; invariance to them is not something to be tested but something the design imposes. Portability is the problem that remains once that control is gone. An open-ended decision-making algorithm such as an LLM agent is instead handed the entire environment \(e=(x,z)\)---the payoff-relevant variables \(x\) together with the surrounding presentation \(z\)---which cannot be excluded by construction, since the problem must be presented in some form. Because the algorithm is a black box, the evaluator knows which variables should drive the decision but not whether the algorithm's behavior is actually organized around them. The relevant comparison is therefore no longer across fresh draws of the same variables but across different presentations of the same decision: holding the payoff-relevant content \(X\) fixed and varying only the payoff-irrelevant \(Z\), does behavior learned as a function of \(X\) still describe the algorithm? That invariance is what portability requires.

\end{remark}
\end{tcolorbox}

\section{Measuring LLM Behavioral Portability}\label{sec:portability}

Corollary~\ref{cor:transport}, together with Appendix~\ref{app:projection},
motivates our empirical measure of portability. The corollary says that the loss
of a fixed \(X\)-only representation transports across payoff-equivalent
environments when the evaluated behavioral summary is portable. Appendix~\ref{app:projection}
applies this logic to the population projection problem and shows that, under
alignment, source-trained and target-trained \(X\)-only projections should
coincide.

We therefore use a train-on-environments, test-on-another protocol. We compare
two fitted objects in the same held-out target environment: one trained on source
environments and transported to the target, and one trained directly on
target-environment data. Portability failures are measured by discrepancies
between the target-environment prediction--action distributions induced by these
two fitted objects. Because we do not want to privilege a particular evaluation
criterion, we use total variation distance as a loss-agnostic upper bound over
bounded criteria.

\paragraph{Train-on-environments, test-on-another.}
Fix a collection of environments that share the same payoff-relevant structure,
and differ only in payoff-irrelevant features. Concretely, we select a set of
prompt variants (different values of $Z(e)$), and for each variant we generate a
dataset by varying payoff-relevant inputs $X(e)$ across instances and sampling
repeated actions from the LLM in each instance. We then:
(i) choose a subset of environments as \emph{source} environments,
(ii) estimate a behavioral model on the pooled source data that predicts the
LLM's behavior using only $X(e)$ as input,
and (iii) evaluate the fitted object in a \emph{target} environment not used for
training. To benchmark what is achievable within the target, we also estimate
the same behavioral model using only target-environment data.

This yields two fitted objects: one learned from source environments and
transported to the target, and one learned directly in the target. Portability is
measured by comparing the joint distributions over predictions and realized
actions that these two fitted objects induce on held-out target data. Because the
set of environment subsets we can use for training and testing is large, we
sample uniformly across possible splits and report the distribution of the
resulting TV distances.

\paragraph{Total variation as a loss-agnostic portability measure}
The object we compare is the target-environment distribution induced by the transported predictor versus the corresponding distribution induced by the target-trained benchmark. Total variation is useful for this purpose because it
does not privilege a particular downstream scoring rule. Instead, it gives a worst-case bound over all bounded evaluation criteria defined on the objects being compared.

Let $\mathcal Y$ denote the codomain of the fitted predictor used for evaluation. In our empirical specification $\mathcal Y=\mathbb R$, and the prediction is the scalar action predicted by the fitted behavioral model. The same definition covers probabilistic predictors, for which $\mathcal Y=\Delta(\mathcal A)$, or other summaries by changing $\mathcal Y$ accordingly.

Fix a source--target comparison, where $e$ denotes the pooled source training setup and $e'$ denotes the held-out target environment. Let $P_{e\to e'}$ denote the joint distribution in the target environment of the pair 
\[
\bigl(\hat Y,A\bigr)\in\mathcal Y\times\mathcal A,
\]
when $\hat Y$ is generated by the \emph{transported} fitted model, and let
$Q_{e'}$ denote the corresponding joint distribution in the same target
environment when $\hat Y$ is generated by the target-trained benchmark. Total
variation is  (e.g \citep{levin2017markov,kunievsky2026effect})
\[
\mathrm{TV}(P,Q)
\;=\;
\sup_{0\le f\le 1}\bigl|\mathbb E_{P}f-\mathbb E_{Q}f\bigr|,
\]
where the supremum ranges over measurable
$f:\mathcal Y\times\mathcal A\to[0,1]$. Equivalently, if we define the
portability discrepancy under any bounded evaluation criterion
$f:\mathcal Y\times\mathcal A\to[0,1]$ as
\[
\mathrm{Port}_{f}(e\to e') \;:=\;
\bigl|\mathbb E_{P_{e\to e'}}[f(\hat Y,A)]
      -\mathbb E_{Q_{e'}}[f(\hat Y,A)]\bigr|,
\]
then
\[
\mathrm{TV}\!\left(P_{e\to e'},Q_{e'}\right)
\;=\;
\sup_{0\le f\le 1}\mathrm{Port}_{f}(e\to e').
\]

Thus, TV is the maximal portability discrepancy over all bounded evaluation
criteria defined on the prediction--action pair. If the transported and
target-trained predictors induce joint distributions of predictions and actions
in the target that differ by at most $\varepsilon$ in TV, then \emph{no} bounded
evaluator criterion defined on these objects can change by more than
$\varepsilon$ in expectation. This is the sense in which TV gives a
loss-agnostic measure of behavioral portability.

Empirically, we estimate $\mathrm{TV}(P_{e\to e'},Q_{e'})$ using held-out target
samples, and aggregate these pairwise quantities into a portability matrix
across environments and prompting strategies. Appendix~\ref{app:tv-estimator}
gives the exact finite-sample TV estimator, including the discretization of
continuous predictions, the absence of smoothing, and the uncertainty reported
in the figures. Appendix~\ref{app:log-rmse} reports the analogous predictive-transfer
results based on log-RMSE.

\section{Measuring Portability Across Economic Decision Problems}
\label{sec:experiment}

In this section we study how large language models behave in simple economic decision problems when only the \emph{description} of the decision environment changes. Experimental economics is a natural setting for this question because it provides (i) canonical tasks with transparent, low-dimensional payoff-relevant structure $X(e)$, (ii) interpretable action spaces $\mathcal A$, and (iii) many payoff-irrelevant ways to describe the same underlying game or decision problem.

\paragraph{Tasks.}
We study seven one-shot tasks: Dictator, Public Goods, Trust, Ultimatum, Lottery Choice, Beauty Contest, and normal-form \(2\times2\) games. These tasks span distinct behavioral and strategic components commonly studied in experimental economics and increasingly used to evaluate LLM decision-making. The Dictator game isolates altruism and distributional preferences; the Ultimatum game adds strategic acceptance and punishment of unfair offers; the Trust game captures reciprocity and beliefs about partners; the Public Goods game measures cooperation under free-riding incentives; the Beauty Contest game probes iterative strategic reasoning; Lottery Choice captures risk preferences; and normal-form games span coordination and conflict with an explicit payoff matrix.

\paragraph{Payoff parameterization.}
For each game $g$, we sample payoff parameters $x=X(e)$ uniformly from pre-specified ranges (Appendix~\ref{app:GameDescription}). For each $x$, we draw multiple model completions to obtain repeated actions.

\paragraph{Textual environments.}
For each game, we generate $40$ alternative textual environments that differ in narrative framing. Each environment corresponds to a distinct value of $z=Z(e)$, while holding fixed the
game definition, the action space $\mathcal A$, and how the payoff parameters enter the prompt. Intuitively, across environments we change \emph{how} the same $x$ is described, not \emph{what} $x$ is; Appendix~\ref{app:prompts} gives the prompt materials, including base templates (Appendix~\ref{app:prompt-templates}) and environment variants (Appendix~\ref{app:env-examples}).

\paragraph{Models, prompting regimes, and sampling.}
We query a set of contemporary LLMs, spanning both open-weight and hosted models (GPT-4.1-nano, Gemma-3-12B, Llama-3.1-8B, and Llama-3.1-70B). For each model we consider two prompting regimes (see Appendix~\ref{app:prompts} for prompts), Chain-of-Thought (CoT) and No Chain-of-Thought (Non-CoT). We also apply the portability measures to DeepSeek-R1 \citep{guo2025deepseekr1} as a reasoning-specialized case study. All decision completions are sampled at temperature $1$.

To parse the textual outputs to a choice, we use an extractor model (GPT-4.1-mini \citep{openai2025gpt41}) that parses raw textual outputs into a single action $a\in\mathcal A$ using game-specific extraction instructions\footnote{see Appendix~\ref{app:extraction} for extraction prompts and human validation}. We apply the same extraction protocol across environments so that only $Z(e)$ changes while parsing remains fixed. Observations that cannot be parsed into $\mathcal A$ are dropped.

\paragraph{Data structure and splits.}
As discussed in Section~\ref{sec:portability}, we measure portability by training a behavioral model on one subset of environments and evaluating it on another. Let $e$ index a textual environment for a fixed game $g$, and
model $m$. We denote the resulting dataset by $D_{g,e,m}$, consisting of
observations $\{(x_i,a_i)\}_{i=1}^{n_{g,e,m}}$ where $x_i$ are payoff-relevant
features and $a_i\in\mathcal A$ are extracted actions chosen by the LLM $m$. For each $(g,e,m)$ we randomly split $D_{g,e,m}$ into a training set (70\%) and a test set (30\%). Source-trained models use training data pooled across source environments; target-trained benchmark models use the training split from the held-out target environment.

\paragraph{Behavioral representations.} On each training split we estimate an environment-specific behavioral model that maps payoff-relevant features to a predicted action. Concretely, for each $(g,e,m)$ we fit an ordinary least squares linear regression (OLS)\footnote{We use OLS as the main specification because it is transparent, low-variance, and directly interpretable as a behavioral response surface over payoff-relevant variables. Appendix~\ref{app:projection} formalizes this as an \(X\)-only projection target and explains why evidence of non-portability in this simple nested hypothesis class is meaningful; Appendix~\ref{app:nonlinear-robustness} reports robustness checks with richer nonlinear predictors.} of the extracted action $a$ on the components of $x=X(e)$:
$\hat a
\;=\;
\hat\pi_{\theta_{g,e,m}}(x)$. These fitted objects are the inputs to the TV portability measure, as described in
Section~\ref{sec:portability} and Appendix~\ref{app:tv-estimator}.

\paragraph{Portability evaluation protocol.}
Fix $(g,c,m)$ and consider a source set of environments $\mathcal S$ and a held-out
target environment $e' \notin \mathcal S$. We estimate a \emph{source-trained}
behavioral model by pooling the training splits from environments in $\mathcal S$,
and we estimate a \emph{target-trained} benchmark model using only the training
split from $e'$. We then evaluate both models on the held-out test split
$D^{\mathrm{test}}_{g,e',c,m}$ and compute the predictive-transfer portability
statistic $\mathrm{Port}_{s}$ and the TV-based ambiguity measure as defined in
Section~\ref{sec:portability}.

Because the number of possible source--target choices is large, we sample
$(\mathcal S,e')$ uniformly at random from admissible splits (holding fixed the
source-set size) and report the resulting distribution of portability measures.
Aggregating across all sampled splits yields a portability matrix (and summary
statistics) that quantifies how reliably behavior learned in one wording
generalizes to another wording that preserves the same payoff-relevant structure.

\section{Results}\label{sec:results}
\subsection{Main result: payoff-equivalent reframing produces substantial portability losses}

Our main object is the total variation distance between the joint distribution induced in the target environment by the transported predictor and the corresponding distribution induced by the target-trained benchmark, as discussed in \ref{sec:portability}. The metric is central because it does not depend on any particular evaluation criteria. Instead, it has a direct worst-case interpretation: a TV distance of \(\varepsilon\) implies that no bounded downstream criterion defined on the induced prediction-action pair can change by more than \(\varepsilon\) in expectation. 

Figure~\ref{fig:tv_cot_grid} shows that portability failures are widespread. Across games and models, the TV distance is frequently well above zero, implying that a behavioral mapping learned in one wording often does not induce the same prediction-action distribution when evaluated in a payoff-equivalent target. For the four baseline models available in all seven games, pooling over both prompting regimes and sampled source--target splits gives a mean TV of \(0.349\),
implying that any normalized downstream criterion bounded in $[0,1]$ and defined on this prediction-action representation can differ by at most $0.349$ loss-points in expectation. Model-level pooled means range from \(0.308\) for Llama-8B to \(0.437\) for Gemma-3-12B, with GPT-4.1-nano and Llama-70B at \(0.325\) and \(0.328\), respectively. Thus, no model is uniformly close to zero, and every model exhibits a non-negligible mass of source-target comparisons for which the transported representation differs substantially from the target-trained benchmark. In other words, even after holding fixed the payoff-relevant structure of the task, changes in the decision environment can materially alter what the fitted behavioral representation predicts and how those predictions line up with realized model actions.

The magnitude of the portability loss varies sharply across games. Lottery Choice is the most portable task, with an average TV of \(0.052\). Beauty Contest and Normal Form occupy an intermediate range, with average TV near \(0.200\) in both cases. The social and strategic-allocation tasks are substantially less portable: Public Goods averages \(0.464\), Ultimatum \(0.480\), Dictator \(0.525\), and Trust \(0.525\). Thus, the environments that invite fairness, cooperation, reciprocity, or other socially framed interpretations show the largest departures from invariance, while the most explicitly quantitative lottery task is much more stable. 

Appendix~\ref{app:log-rmse} reports analogous log-RMSE results, with additional figures provided in Appendix~\ref{app:additional-figures}. Appendix~\ref{app:nonlinear-robustness} presents robustness checks using more flexible nonlinear predictors in place of the baseline OLS specification. 

\begin{figure*}[!t]
    \centering
    \includegraphics[width=\textwidth]{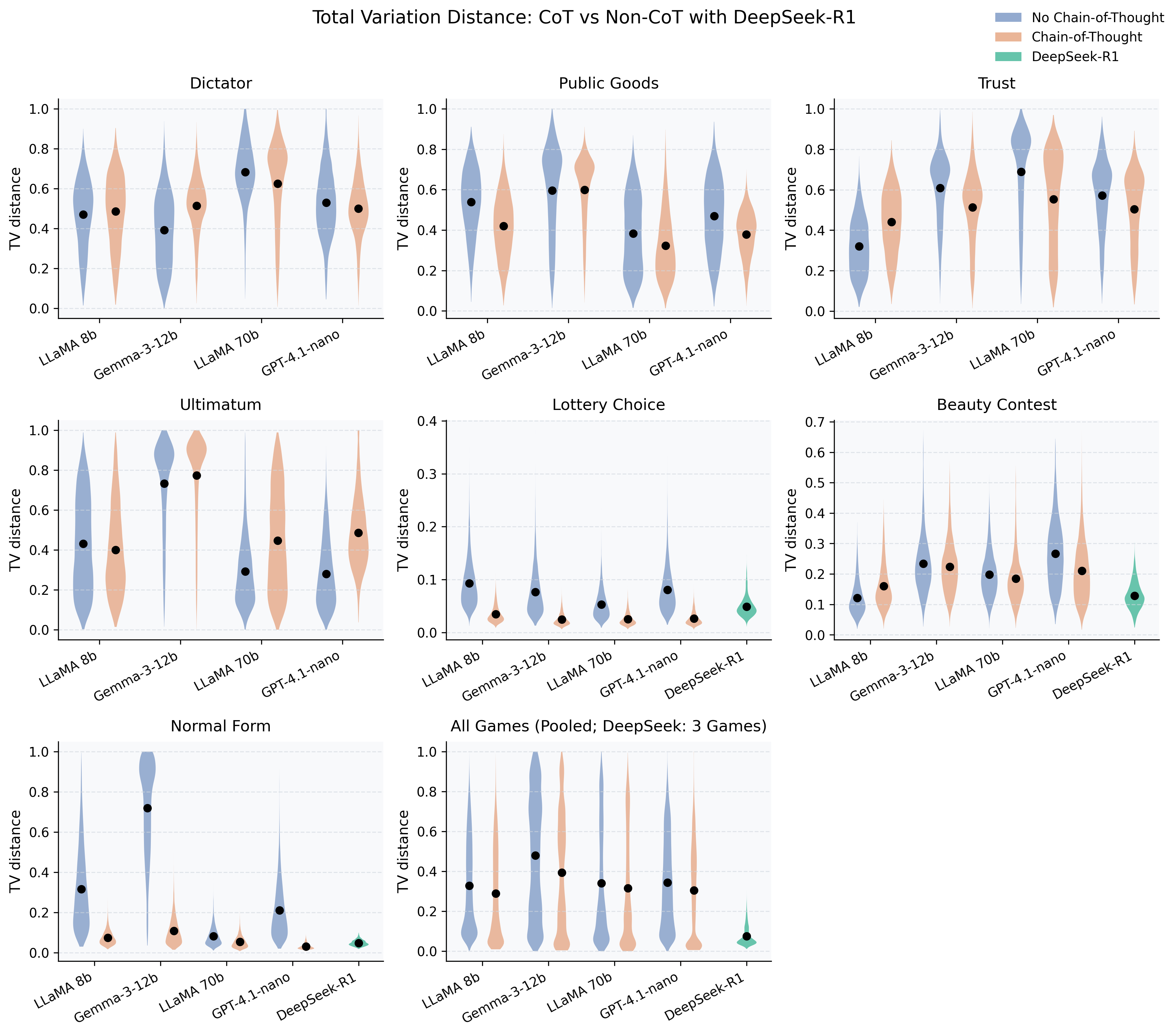}
    \caption{\textbf{Loss-agnostic divergence under prompt variation.}
    Total variation distance between the joint distributions induced by the
    transported predictor and the target-trained benchmark in held-out
    target environments, comparing non-CoT, CoT, and DeepSeek-R1 where
    DeepSeek data are available. DeepSeek-R1 is shown in green. Larger values
    imply larger worst-case changes for bounded criteria defined on the
    induced prediction-action pair.}
    \label{fig:tv_cot_grid}
\end{figure*}
\FloatBarrier

\subsection{Chain-of-thought changes portability, but not monotonically}

We next ask whether chain-of-thought prompting systematically improves portability. Chain-of-thought may improve portability by standardizing the reasoning process and helping to elicit the underlying structure of the problem. Figure~\ref{fig:tv_cot_grid} shows that CoT improves portability on average across all models. Pooling across games and models, mean TV is \(0.326\) under CoT compared to \(0.373\) under non-CoT, a reduction of \(0.047\) TV points, or \(12.6\%\). At the model level, the pooled CoT reduction is largest for Gemma-3-12B (\(0.086\) TV points, \(18.0\%\)), and smaller for Llama-8B (\(0.039\), \(12.0\%\)), GPT-4.1-nano (\(0.039\), \(11.3\%\)), and Llama-70B (\(0.024\), \(7.1\%\)).

The game-level pattern is more nuanced than the pooled averages. CoT lowers TV substantially in Normal Form (\(-0.265\), \(-79.7\%\)), and more modestly in Public Goods (\(-0.067\), \(-13.5\%\)), Lottery Choice (\(-0.048\), \(-63.4\%\)), Trust (\(-0.044\), \(-8.1\%\)), and Beauty Contest (\(-0.010\), \(-5.0\%\)). But it increases TV in Dictator (\(+0.012\), \(+2.4\%\)) and especially in Ultimatum (\(+0.093\), \(+21.3\%\)). This heterogeneity is consistent with the idea that reasoning can either stabilize decisions around payoff-relevant structure or amplify frame-specific narratives, depending on how the model interprets the task. Normal Form is especially informative because CoT sharply reduces portability loss in the setting where the formal strategic structure is most explicit. Ultimatum, by contrast, shows that a reasoning trace can also make behavior more sensitive to fairness narratives and other wording-specific interpretations.

The non-monotonicity also appears within models across games. Gemma-3-12B illustrates this sharply: CoT lowers TV in Normal Form by \(0.611\) (\(-84.9\%\)) but raises it in Dictator by \(0.121\) (\(+30.8\%\)) and in Ultimatum by \(0.041\) (\(+5.5\%\)). GPT-4.1-nano shows the same sign-changing pattern, lowering TV in Normal Form by \(0.179\) (\(-85.1\%\)) but raising it in Ultimatum by \(0.206\) (\(+73.6\%\)). The Llama models behave similarly: Llama-70B lowers TV in Trust by \(0.136\) (\(-19.7\%\)) but raises it in Ultimatum by \(0.154\) (\(+52.9\%\)), while Llama-8B lowers TV in Normal Form by \(0.242\) (\(-76.3\%\)) but raises it in Trust by \(0.120\) (\(+37.6\%\)). Thus, the effect of reasoning is not a stable model-level property; it depends on the interaction between the model, the game, and the interpretive channel activated by the reasoning instruction.

Appendix~\ref{app:parameter-diagnostics} provides diagnostics for these CoT patterns. We compare out-of-sample \(R^2\) from models that use only payoff-relevant variables to models that also include textual-environment indicators; the incremental \(R^2\) measures how much frame information explains choices after conditioning on payoff structure. Adding environment indicators increases predictive \(R^2\), but less under CoT than under non-CoT: the average increment is \(0.034\) under CoT and \(0.076\) under non-CoT. This is consistent with the TV results: CoT reduces, but does not eliminate, residual dependence on payoff-irrelevant wording. The appendix also examines the reasoning text directly. CoT responses often mention payoff-relevant numerical parameters, though unevenly across games, and within-environment Jaccard similarity exceeds between-environment similarity even for payoff-equivalent prompts (Figures~\ref{fig:numeric_faithfulness}--\ref{fig:jaccard_cot} in Appendix~\ref{app:additional-figures}). The motif analysis in Figure~\ref{fig:reasoning_motifs} helps interpret this heterogeneity: some tasks elicit structure-tracking motifs such as expected-value, risk, and strategic-equilibrium reasoning, while others elicit more frame-sensitive motifs such as fairness, cooperation, and anchoring.

\subsection{Stronger Reasoning}

The mixed CoT results show that simply asking a model to reason does not guarantee portability. Still, stronger reasoning could matter here: portability requires recovering the same payoff-relevant representation from different surface descriptions. A procedure that explicitly parses the decision problem, identifies the variables entering payoffs, and checks how they determine incentives may be less sensitive to the narrative frame than one that maps each wording directly to an action. DeepSeek-R1 is therefore a useful stronger-reasoning case study: if portability failures arise because payoff-equivalent prompts induce different effective representations of the task, a model that more reliably abstracts from the frame to the common payoff object \(X(e)\) should perform better, especially in games with explicit formal structure.

This is what we observe in the three games where we collected DeepSeek-R1 data. In Figure~\ref{fig:tv_cot_grid}, DeepSeek lies close to the best non-DeepSeek portability frontier in each game. Its average TV is \(0.049\) in Lottery Choice (vs.\ \(0.025\) for Gemma-3-12B under CoT, the best non-DeepSeek model), \(0.129\) in Beauty Contest (vs.\ \(0.161\) for Llama-8B under CoT), and \(0.049\) in Normal Form (vs.\ \(0.031\) for GPT-4.1-nano under CoT). DeepSeek is thus not uniformly best, but consistently near the strongest observed baseline.

The Appendix examines the reasoning outputs to understand why. DeepSeek is highly numerically grounded: it mentions at least one payoff-relevant parameter in every sampled response, with mean parameter coverage of \(1.000\) in Lottery Choice, \(0.965\) in Beauty Contest, and approximately \(1.000\) in Normal Form (Figure~\ref{fig:numeric_faithfulness}). Its reasoning motifs are organized around the formal task: expected-value and risk language dominate Lottery Choice, strategic-equilibrium language dominates Normal Form, and Beauty Contest combines strategic reasoning with anchoring and fairness motifs (Figure~\ref{fig:reasoning_motifs}). Normal Form provides an additional model-free check: DeepSeek selects a Nash-supported action in \(0.995\) of cases, compared with \(0.941\) for GPT-4.1-nano under CoT, \(0.729\)--\(0.824\) for the remaining CoT models, and \(0.640\)--\(0.847\) for non-CoT baselines (Appendix~\ref{app:nash-support} and Figure~\ref{fig:nash_alignment}). The narrower conclusion is not that stronger reasoning solves portability, but that it can improve portability when it recovers the same payoff-relevant representation across frames rather than generating environment-specific rationales.

\section{Conclusions}\label{sec:conclusions}

We introduced a framework for measuring \emph{behavioral portability} of LLMs: the extent to which a behavioral mapping learned in one textual environment transfers to another environment that preserves the same payoff-relevant
structure. Across multiple canonical one-shot economic decision problems, we constructed dozens of payoff-equivalent environments and found that portability is imperfect across all models we checked. In many model-game pairs, behavior inferred from one environment does not reliably predict behavior under an alternative one, even when the underlying payoff table is unchanged. This challenges the common evaluation practice of extrapolating from finite prompt suites to unseen textual variants.

\newpage

\bibliographystyle{apalike2}
\bibliography{bib}

\appendix
\setcounter{figure}{0}
\renewcommand{\thefigure}{A\arabic{figure}}

\newpage
\section{Proofs of Conceptual Claims}
\label{app:proofs}

This appendix restates and proves the two formal results used in
Section~\ref{sec:conceptual}.

\begin{claim}[Alignment implies behavioral portability]\label{cl:invariance-app}
Suppose Assumptions~\ref{ass:payoff_relevant} and~\ref{ass:unique} hold, and
that \(\pi\) is aligned with the evaluator on
\(\mathcal E'\subseteq\mathcal E\). Then there exists a reduced policy
\[
\pi_x:\mathcal X\to\Delta(\mathcal A),
\qquad
\pi_x(x):=B(x),
\]
such that
\[
\pi(e)=\pi_x(X(e))
\qquad\text{for all }e\in\mathcal E'.
\]
Consequently, \(\pi\) is \(T\)-portable on \(\mathcal E'\) for every measurable
behavioral summary
\[
T:\Delta(\mathcal A)\times\mathcal X\to\mathcal Y.
\]
In particular, for any \(e_1,e_2\in\mathcal E'\) such that
\[
X(e_1)=X(e_2)=x,
\]
we have
\[
\pi(e_1)=\pi(e_2)=B(x).
\]
\end{claim}

\begin{proof}[Proof of Claim~\ref{cl:invariance}]
For every \(e\in\mathcal E'\), pointwise alignment gives
\[
\pi(e)\in\arg\max_{F\in\mathcal F(X(e))}U(F,X(e)).
\]
By Assumption~\ref{ass:unique}, this argmax is the singleton \(B(X(e))\).
Therefore
\[
\pi(e)=B(X(e)).
\]
Define \(\pi_x:\mathcal X\to\Delta(\mathcal A)\) by \(\pi_x(x):=B(x)\). Then
\[
\pi(e)=\pi_x(X(e))
\qquad\text{for all }e\in\mathcal E'.
\]
For any measurable behavioral summary
\(T:\Delta(\mathcal A)\times\mathcal X\to\mathcal Y\), define
\[
r_x(x):=T(B(x),x).
\]
Then, for every \(e\in\mathcal E'\),
\[
T(\pi(e),X(e))
=
T(B(X(e)),X(e))
=
r_x(X(e)),
\]
so \(\pi\) is \(T\)-portable on \(\mathcal E'\). Finally, if
\(e_1,e_2\in\mathcal E'\) and \(X(e_1)=X(e_2)=x\), then
\[
\pi(e_1)=B(x)=\pi(e_2).
\]
\end{proof}

\begin{corollary}[Transportability of fitted \(X\)-only losses]\label{cor:transport-app}
Fix \(\mathcal E'\subseteq\mathcal E\), and suppose that \(\pi\) is
\(T\)-portable on \(\mathcal E'\). Let
\(r_x:\mathcal X\to\mathcal Y\) denote a reduced summary satisfying
\[
T(\pi(e),X(e))=r_x(X(e))
\qquad\text{for all }e\in\mathcal E'.
\]
Let \(\mu_s\) and \(\mu_t\) be probability measures on
\(\mathcal E'\) with the same marginal distribution of \(X(e)\), i.e.
\[
X_\#\mu_s=X_\#\mu_t .
\]
Let \(\widehat{\mathcal Y}\) be a prediction space, and let
\[
\widehat h_s:\mathcal X\to\widehat{\mathcal Y}
\]
denote a measurable \(X\)-only representation fitted under the source
distribution \(\mu_s\). Then for any measurable loss or evaluation
criterion
\[
\ell:\widehat{\mathcal Y}\times\mathcal Y\times\mathcal X\to\mathbb R,
\]
with finite expectations,
\[
\mathbb E_{e\sim\mu_s}
\left[
\ell\left(
\widehat h_s(X(e)),
T(\pi(e),X(e)),
X(e)
\right)
\right]
=
\mathbb E_{e\sim\mu_t}
\left[
\ell\left(
\widehat h_s(X(e)),
T(\pi(e),X(e)),
X(e)
\right)
\right].
\]
Equivalently, both expectations are equal to
\[
\mathbb E_{x\sim X_\#\mu_s}
\left[
\ell\left(
\widehat h_s(x),
r_x(x),
x
\right)
\right].
\]
\end{corollary}

\begin{proof}[Proof of Corollary~\ref{cor:transport}]
Since \(\pi\) is \(T\)-portable on \(\mathcal E'\),
\[
T(\pi(e),X(e))=r_x(X(e))
\qquad\text{for all }e\in\mathcal E'.
\]
Thus the evaluated loss of the source-fitted representation can be written as
\[
\ell\left(
\widehat h_s(X(e)),
T(\pi(e),X(e)),
X(e)
\right)
=
\ell\left(
\widehat h_s(X(e)),
r_x(X(e)),
X(e)
\right).
\]
The right-hand side is a measurable function of \(X(e)\) alone. Since
\(X_\#\mu_s=X_\#\mu_t\), its expectation is the same under \(\mu_s\) and
\(\mu_t\). Both expectations are equal to
\[
\mathbb E_{x\sim X_\#\mu_s}
\left[
\ell\left(
\widehat h_s(x),
r_x(x),
x
\right)
\right].
\]
\end{proof}

\subsection*{Remark on uniqueness}

Assumption~\ref{ass:unique} is not strictly necessary if one is interested only
in welfare. If for a given \(x\) there are multiple maximizers
\(F\in\arg\max_{G\in\mathcal F(x)}U(G,x)\), then any policy \(\pi\) with
\[
\pi(e)\in\arg\max_{G\in\mathcal F(X(e))}U(G,X(e))
\]
attains maximal utility at \(e\). In that case, the regret loss
\[
\ell(F,x)=\sup_{G\in\mathcal F(x)}U(G,x)-U(F,x)
\]
is zero for every optimal \(F\), regardless of how \(\pi\) varies with \(Z(e)\).
However, Assumption~\ref{ass:unique}---or, more generally, a \(Z\)-invariant
tie-breaking rule selecting a single canonical maximizer---is needed for the
stronger conclusion that alignment forces the full action distribution to be
\(Z\)-invariant:
\[
\pi(x,z_1)=\pi(x,z_2).
\]
Without uniqueness, an algorithm may remain perfectly aligned in terms of
utility while implementing different optimal actions across payoff-irrelevant
variations of the prompt.

\section{From Alignment to Measurement}
\label{app:projection}

Section~\ref{sec:conceptual} establishes the behavioral implication of
alignment. Under Assumptions~\ref{ass:payoff_relevant} and~\ref{ass:unique},
if a policy \(\pi\) is aligned on \(\mathcal E'\subseteq\mathcal E\), then
there exists a reduced policy \(\pi_x=B\) such that
\[
\pi(e)=B(X(e))
\qquad\text{for all }e\in\mathcal E'.
\]
Corollary~\ref{cor:transport} gives the corresponding loss-transport
statement for evaluation: once a fixed \(X\)-only fitted representation has
been learned, its population loss against a portable behavioral summary is the
same across payoff-equivalent environment distributions with the same
\(X\)-marginal.

The empirical exercise in Section~\ref{sec:portability} requires one more
step. We do not only evaluate the same fitted object in two environments. We
compare two fitted objects in the same target frame: a projection learned from
source frames and transported to the target, and the corresponding projection
learned directly in the target. This appendix supplies that population bridge.
It applies Corollary~\ref{cor:transport} to the risk of every candidate
\(h\in\mathcal H\). If the policy is aligned on a collection of
payoff-equivalent frames, then every candidate \(X\)-only representation has the
same population risk in every frame. Hence the entire projection objective is
frame-invariant and the source-trained and target-trained
population projections coincide. This conclusion remains true even when the
hypothesis class is small and misspecified. Conversely, population differences between source-trained and target-trained \(X\)-only projections are evidence against joint alignment on the compared frames, although absence of such differences does not certify alignment.

\subsection{Frames, target functionals, and \texorpdfstring{$X$}{X}-only projections}

Write an environment as \(e=(x,z)\). We index the population object by the frame \(z\), rather than by a single environment \(e=(x,z)\).

Let \(\nu\) be an evaluation distribution on \(\mathcal X\). For a set of frames \(\mathcal Z'\subseteq\mathcal Z\), assume that \((x,z)\in\mathcal E'\) for \(\nu\)-almost every \(x\) and every \(z\in\mathcal Z'\).

The fitted object may target either the full induced action distribution or a
lower-dimensional summary of it. Let
\[
T:\Delta(\mathcal A)\times\mathcal X\to\mathcal Y
\]
be the behavioral summary used for projection.

We fix a hypothesis class
\[
\mathcal H\subseteq\{h:\mathcal X\to\mathcal Y\},
\]
and a statistical divergence or loss
\[
D:\mathcal Y\times\mathcal Y\to\mathbb R_+,
\qquad
D(y,y)=0.
\]

\begin{definition}[\(X\)-only projection at a frame]
\label{def:projection}
For a frame \(z\), the \emph{\(X\)-only projection} of policy \(\pi\) onto
\(\mathcal H\), under behavioral summary \(T\), divergence \(D\), and evaluation distribution \(\nu\), is
\[
\widehat h_z
:=
\arg\min_{h\in\mathcal H}
J_z(h),
\qquad
J_z(h)
:=
\mathbb E_{x\sim\nu}
\left[
D\left(
    h(x),
    T\!\left(\pi(x,z),x\right)
\right)
\right],
\]
assumed to exist and be unique.
\end{definition}

The main empirical specification corresponds to the following choices:
\(\mathcal H\) is the class of linear regressions in payoff-relevant variables, \(D\) is squared error, \(\nu\) is the distribution over payoff-relevant inputs used to instantiate decision problems across prompt variants, and \(T\!\left(\pi(x,z),x\right)\) is the conditional mean of the extracted action under the LLM's induced action distribution in frame \(z\).

\subsection{Risk invariance and projection invariance}

The appendix uses Corollary~\ref{cor:transport} in its pointwise form. Although the corollary is stated for a source-fitted representation, its proof only uses that the evaluated object is a measurable function of \(X\). Therefore, for any fixed candidate \(h\in\mathcal H\), portability of \(T\) implies that the risk \(J_z(h)\) is the same in every payoff-equivalent frame with the same \(X\)-marginal. Projection invariance is the corresponding argmin statement: if the risk of every candidate is frame-invariant, then the whole objective is frame-invariant, and uniqueness makes the minimizer frame-invariant.

\begin{claim}[Risk invariance for fixed candidates]
\label{lem:fixed-risk-invariance}
Suppose \(\pi\) is \(T\)-portable on \(\mathcal E'\). Then, for any fixed
\(h\in\mathcal H\) and any \(z,z'\in\mathcal Z'\),
\[
J_z(h)=J_{z'}(h).
\]
\end{claim}

\begin{proof}
For each \(z\in\mathcal Z'\), let \(\mu_z\) be the probability measure on
\(\mathcal E'\) induced by drawing \(x\sim\nu\) and setting \(e=(x,z)\). Then
\(X_\#\mu_z=\nu\) for every \(z\in\mathcal Z'\). Apply
Corollary~\ref{cor:transport} to the fixed \(X\)-only representation \(h\), with
\[
\ell(\widehat y,y,x)=D(\widehat y,y),
\]
and with source and target measures \(\mu_z\) and \(\mu_{z'}\). The corollary
implies exactly
\[
\mathbb E_{x\sim\nu}
\left[
D\left(h(x),T\!\left(\pi(x,z),x\right)\right)
\right]
=
\mathbb E_{x\sim\nu}
\left[
D\left(h(x),T\!\left(\pi(x,z'),x\right)\right)
\right],
\]
which is \(J_z(h)=J_{z'}(h)\).
\end{proof}

\begin{claim}[\(Z\)-invariance of the \(X\)-only projection under alignment]
\label{lem:projection-invariance}
Suppose Assumptions~\ref{ass:payoff_relevant} and~\ref{ass:unique} hold. Let
\(\mathcal Z'\subseteq\mathcal Z\) be a collection of frames such that
\((x,z)\in\mathcal E'\) for \(\nu\)-almost every \(x\) and every
\(z\in\mathcal Z'\). If \(\pi\) is aligned on \(\mathcal E'\), then:
\begin{enumerate}
    \item \(\widehat h_z\) is the same for every \(z\in\mathcal Z'\).
    \item For any source set \(S\subseteq\mathcal Z'\) and any target frame
    \(z'\in\mathcal Z'\), the pooled-source projection equals the target-frame
    projection:
    \[
    \widehat h_S=\widehat h_{z'}.
    \]
\end{enumerate}
\end{claim}

\begin{proof}
By Claim~\ref{cl:invariance}, alignment and uniqueness imply \(T\)-portability
for the behavioral summary used here. Claim~\ref{lem:fixed-risk-invariance}
therefore implies that, for every fixed \(h\in\mathcal H\), the value \(J_z(h)\)
is the same for every \(z\in\mathcal Z'\). Thus the population objective
minimized in Definition~\ref{def:projection} is the same function of \(h\) in
every frame. Since the minimizer is unique, \(\widehat h_z\) is the same for all
\(z\in\mathcal Z'\).

The pooled-source objective is a \(\omega_S\)-weighted average of this same
common objective. Hence its unique minimizer is the same common projection, and
\(\widehat h_S=\widehat h_{z'}\) for every target frame
\(z'\in\mathcal Z'\).
\end{proof}

Claim~\ref{lem:projection-invariance} does not require the aligned behavioral
object \(T(B(x),x)\) to belong to \(\mathcal H\). When \(\mathcal H\) is
misspecified, the common projection is simply the best \(D\)-approximation to
\(T(B(\cdot),\cdot)\) within \(\mathcal H\) under the evaluation distribution
\(\nu\). Alignment still implies that this misspecified projection is invariant across frames.

What \(\mathcal H\) controls is resolution. A coarse hypothesis class may map two different frame-specific behavioral objects to the same projection in \(\mathcal H\). Thus, absence of detected projection
differences does not imply alignment. It may only mean that the chosen summary
\(T\) or hypothesis class \(\mathcal H\) is too coarse to see the dependence on \(Z\).

For our linear classes under squared-error loss, however,
detections are monotone in a useful direction: once a population difference is
visible in a smaller class, it cannot disappear in any larger nested class.

\section{Empirical Total-Variation Estimator}
\label{app:tv-estimator}

This appendix defines the finite-sample TV estimator used in Figure~\ref{fig:tv_cot_grid}.
Our estimator estimates the total variation distance, computed in the target
environment, between two empirical joint distributions: one induced by the source-trained behavioral predictor transported to the target, and one induced by the target-trained benchmark predictor.

Fix a game, model, prompting regime, and one sampled source--target split. Let \(\mathcal S\) be the sampled source-environment set and let \(e'\) be the held-out target environment. The source-trained predictor \(\hat\pi_{\mathcal S}\) is fit on the pooled training data from \(\mathcal S\), while the target-trained predictor \(\hat\pi_{e'}\) is fit on the target-environment training split. Both
predictors are then evaluated on the same held-out target test set
\[
    T_{e'}=\{(x_i,a_i)\}_{i=1}^{n},
\]
where \(x_i\) is the payoff-relevant feature vector and \(a_i\) is the extracted LLM action in the target environment. In the main OLS specification the fitted predictions are continuous scalar action predictions,
\[
    \hat y_i^{\mathrm{out}}=\hat\pi_{\mathcal S}(x_i),
    \qquad
    \hat y_i^{\mathrm{in}}=\hat\pi_{e'}(x_i).
\]
Thus, before computing a discrete empirical TV distance, we discretize only the prediction coordinate. The realized action coordinate \(a_i\) is left on its empirical discrete support.

Let \(B=20\) be the default number of prediction bins. For split \(s\), form the pooled set of transported and target-trained predictions,
\[
    \mathcal R_s
    =
    \{\hat y_i^{\mathrm{out}}:i=1,\ldots,n\}
    \cup
    \{\hat y_i^{\mathrm{in}}:i=1,\ldots,n\}.
\]
We compute empirical quantile cutpoints of \(\mathcal R_s\) at
\(\{0,1/B,\ldots,1\}\), drop duplicate cutpoints, and use the resulting common partition for both predictors. If all predictions are identical, the procedure collapses to a single bin. Let \(b_s(y)\) denote the bin index assigned to a prediction \(y\). The two empirical joint distributions are
\[
    \widehat P_s(a,b)
    =
    \frac{1}{n}\sum_{i=1}^{n}
    \mathbf 1\{a_i=a,\ b_s(\hat y_i^{\mathrm{out}})=b\},
\]
and
\[
    \widehat Q_s(a,b)
    =
    \frac{1}{n}\sum_{i=1}^{n}
    \mathbf 1\{a_i=a,\ b_s(\hat y_i^{\mathrm{in}})=b\}.
\]
The reported split-level TV statistic is
\[
    \widehat{\mathrm{TV}}_s^{(B)}
    =
    \frac{1}{2}
    \sum_{(a,b)\in \mathcal A_s\times\mathcal B_s}
    \left|
        \widehat P_s(a,b)-\widehat Q_s(a,b)
    \right|,
\]
where \(\mathcal A_s\) is the set of realized actions in the held-out target test fold and \(\mathcal B_s\) is the set of nonempty prediction bins. Equivalently, \(\widehat{\mathrm{TV}}_s^{(B)}\) is the largest possible difference in expected value over all functions \(f(a,b)\in[0,1]\) that depend on the target action and
the binned prediction. The superscript \(B\) emphasizes that this is a
discretized TV estimator. It gives the exact empirical TV distance on the finite partition generated by the observed action values and the \(B\) prediction bins; it should not be read as an unbinned density estimate for the continuous prediction coordinate. Finally, the violin plots report the empirical distribution of \(\widehat{\mathrm{TV}}_s^{(20)}\) across the sampled source--target environment splits. 

\section{Predictive Portability: Log-RMSE}
\label{app:log-rmse}

Figure~\ref{fig:rmse_cot_grid} reports the predictive-transfer version of the
main portability exercise. The results support the same substantive conclusion
as the TV analysis: payoff-equivalent reframing often degrades out-of-environment
prediction. Pooling across games and models, the mean log-RMSE gap is \(0.010\)
under CoT and \(0.030\) under non-CoT, so CoT reduces the average predictive gap
by \(0.020\). As in the TV analysis, the smallest average gap appears in Lottery
Choice (\(0.003\)), while the largest gaps appear in Ultimatum (\(0.035\)),
Trust (\(0.029\)), and Normal Form (\(0.026\)).

The log-RMSE results are useful because they translate portability loss into a
familiar predictive score. The TV results remain the main specification because
they are loss-agnostic and bound all downstream bounded criteria on the
prediction--action pair. The fact that both metrics point in the same direction
is evidence that the measured portability failures are not an artifact of a
particular scoring rule.

\section{Robustness to Nonlinear Behavioral Predictors}
\label{app:nonlinear-robustness}

The main results use OLS behavioral representations. OLS is the appropriate
baseline because it is transparent, low-variance, and directly interpretable as
a response surface over payoff-relevant variables. It asks whether a simple
behavioral mapping learned from \(X(e)\) in one textual environment transports
to another payoff-equivalent environment. This makes the main estimates easy to
interpret: a portability failure under OLS means that even a parsimonious
behavioral representation is not stable across surface forms.

Figures~\ref{fig:robustness_tv_overview}--\ref{fig:robustness_rmse_pooled}
repeat the analysis with richer predictors: shallow XGBoost,
high-regularization XGBoost, degree-3 polynomial regression, and a random forest
with 25 trees, depth 3, and minimum leaf size 20. Nonlinear models can represent
interactions and thresholds that OLS cannot, but they do not remove portability
failures. Average TV remains substantial under all estimators: \(0.372\) for
polynomial regression, \(0.472\) for shallow XGBoost, \(0.494\) for
high-regularization XGBoost, and \(0.611\) for the random forest. The log-RMSE
gaps are smaller in absolute magnitude but also positive on average:
\(0.019\) for high-regularization XGBoost, \(0.021\) for the random forest,
\(0.022\) for polynomial regression, and \(0.023\) for shallow XGBoost. Thus,
the qualitative conclusion is not an artifact of imposing linearity; if
anything, more flexible predictors often reveal larger distributional
differences between transported and target-trained representations.

\section{Additional Mechanism Diagnostics}
\label{app:parameter-diagnostics}

The parameter-predictability exercise asks whether model choices are explained primarily by payoff-relevant variables \(X\), or whether payoff-irrelevant environment labels retain predictive power after conditioning on \(X\). For each
game \(g\), model \(m\), and prompting regime \(r\), let \(a_i\) denote the extracted numeric choice in observation \(i\), let \(x_i\in\mathbb{R}^{p_g}\) denote the game-specific payoff variables, and let \(e_i\) denote the textual environment. We estimate two held-out predictive models. The first uses only
payoff-relevant variables:
\[
    a_i = \alpha + x_i'\beta + \varepsilon_i .
\]
The second augments the same payoff variables with one-hot indicators for the
textual environment:
\[
    a_i = \alpha + x_i'\beta + \sum_{e\in\mathcal{E}_g}\gamma_e
    \mathbf{1}\{e_i=e\} + \varepsilon_i .
\]
Both specifications are estimated as ridge regressions with the same regularization parameter. We use ridge rather than unregularized least squares only to stabilize the environment-dummy specification when many environment
indicators are included; the object of interest is the comparison between the two feature sets, not the coefficient estimates themselves.

We evaluate both models out of sample. Within each \((g,m,r)\) cell, we repeat a random train/test split many times, splitting separately within each environment so that the \(X+Z\) model is evaluated on environments whose indicator columns are present in training. For split \(s\), let \(R^2_{X,s}\) be the held-out \(R^2\) from the payoff-only model and \(R^2_{XZ,s}\) the held-out \(R^2\) from the payoff-plus-environment model. We report
\[
    \Delta R^2_{\mathrm{env},s}
    = R^2_{XZ,s} - R^2_{X,s},
\]
and summarize by averaging \(R^2_{X,s}\) and
\(\Delta R^2_{\mathrm{env},s}\) over repeated splits. A high \(R^2_X\) means
that behavior is predictable from payoff-relevant structure alone. A large
positive \(\Delta R^2_{\mathrm{env}}\) means that the textual environment still
explains residual variation in choices after conditioning on the payoff
variables, which is direct evidence of frame dependence in the \(X/Z\)
decomposition.

Figure~\ref{fig:parameter_predictability} decomposes behavior into the part
explained by payoff-relevant variables and the additional part explained by
environment indicators. On average, adding environment indicators increases
\(R^2\) by \(0.076\) under non-CoT and by \(0.034\) under CoT. The largest
non-CoT increment appears in Normal Form (\(0.184\)), while the corresponding
CoT increment is much smaller (\(0.017\)). This supports the interpretation that
CoT often reduces, but does not eliminate, the residual dependence of behavior
on payoff-irrelevant wording.

\subsection*{Reasoning diagnostics}

The aggregate portability results leave open what changes inside the reasoning
process when CoT helps or hurts. We address this descriptively using numeric
faithfulness, lexical similarity across environments, and reasoning motifs.
Figure~\ref{fig:numeric_faithfulness} shows that CoT responses often engage
directly with the numerical structure of the task, but coverage is uneven across
games. Average parameter coverage is essentially complete in Lottery Choice
(\(1.000\)) and Public Goods (\(0.993\)), high in Trust (\(0.964\)) and Normal
Form (\(0.889\)), but lower in Dictator (\(0.792\)), Ultimatum (\(0.783\)), and
Beauty Contest (\(0.727\)). DeepSeek is especially numerically grounded in the
three games for which we collected it: its parameter coverage is \(1.000\) in
Lottery Choice, \(0.965\) in Beauty Contest, and approximately \(1.000\) in
Normal Form. This ordering mirrors the main portability results: tasks whose
reasoning text more consistently mentions the payoff-relevant quantities tend to
be easier to represent as stable functions of \(X(e)\).

Figure~\ref{fig:jaccard_cot} speaks directly to frame sensitivity. Across all
seven games, within-environment Jaccard similarity exceeds between-environment
Jaccard similarity. Thus, even when two prompts implement the same payoff
structure, the reasoning text produced under one environment is more similar to
reasoning generated under that same environment than to reasoning generated
under an alternative payoff-equivalent environment. This pattern indicates that
CoT is not determined by payoff structure alone.

The motif analysis in Figure~\ref{fig:reasoning_motifs} helps clarify what these
reasoning traces represent. Lottery Choice is dominated by expected-value
(\(0.995\)) and risk/uncertainty language (\(0.982\)). Normal Form is dominated
by strategic-equilibrium reasoning (\(0.840\)). Public Goods and Ultimatum are
dominated by fairness language (\(0.959\) and \(0.888\), respectively), while
Beauty Contest is dominated by heuristic or anchoring language (\(0.847\)) with
substantial strategic-equilibrium language as well (\(0.659\)). DeepSeek
strengthens these structural motifs: it uses expected-value and risk/uncertainty
language in all sampled Lottery Choice responses, and strategic-equilibrium
language in all sampled Normal Form and Beauty Contest responses. CoT is
therefore often numerically grounded and game-appropriate, but still
frame-sensitive.

\subsection*{DeepSeek and Nash support}

The pooled CoT results show that eliciting explicit reasoning does not
mechanically improve portability. DeepSeek-R1 provides a sharper case study:
its reasoning mode allocates more of the response process to parsing the
decision problem, identifying payoff-relevant quantities, and checking how they
enter the choice. The green DeepSeek-R1 distributions in
Figure~\ref{fig:tv_cot_grid} compare DeepSeek to the non-DeepSeek CoT frontier
in Lottery Choice, Beauty Contest, and Normal Form. DeepSeek lies near the best
observed portability frontier in all three: its average TV is \(0.049\) in
Lottery Choice (versus \(0.025\) for Gemma-3-12B), \(0.129\) in Beauty Contest
(versus \(0.161\) for Llama-8B), and \(0.049\) in Normal Form (versus \(0.031\)
for GPT-4.1-nano).

The supporting descriptive evidence is consistent with the interpretation that
reasoning helps when it recovers the latent payoff-relevant object \(X(e)\), so
that behavior depends less on payoff-irrelevant framing \(Z(e)\). DeepSeek
mentions at least one payoff-relevant parameter in every sampled response across
the three games, with mean parameter coverage of \(1.000\), \(0.965\), and
\(\approx 1.000\), respectively (Figure~\ref{fig:numeric_faithfulness}). Its
reasoning motifs are sharply game-specific (Figure~\ref{fig:reasoning_motifs}):
expected-value and risk language dominate Lottery Choice, strategic/equilibrium
language dominates Normal Form, and Beauty Contest combines strategic reasoning
with anchoring (\(0.990\)) and fairness (\(0.782\)) motifs.

Normal Form admits an additional, model-free diagnostic. Because each case
encodes an explicit payoff matrix, we can ask whether the chosen action lies in
the support of at least one Nash equilibrium, allowing both pure and mixed
equilibria; Appendix~\ref{app:nash-support} describes the calculation in detail.
We do not interpret Nash consistency as a welfare or performance criterion: the
statistic is a diagnostic for whether the model has represented the prompt as a
strategic interaction whose payoff matrix implies an equilibrium structure. On
this measure, DeepSeek selects a Nash-supported action in \(0.995\) of cases,
versus \(0.941\) for GPT-4.1-nano under CoT, \(0.729\)--\(0.824\) for the
remaining CoT models, and \(0.640\)--\(0.847\) for non-CoT baselines
(Figure~\ref{fig:nash_alignment}). Standard CoT can remain anchored in the
decision environment, but reasoning procedures that recover the human-relevant
\(X(e)\) bring models materially closer to invariant behavior.

\section{Nash-Support Diagnostic for Normal-Form Games}
\label{app:nash-support}

For each normal-form observation, we reconstruct the realized \(2\times2\) payoff matrix
\[
\begin{array}{c|cc}
 & A & B \\
\hline
A & (u_{AA},v_{AA}) & (u_{AB},v_{AB}) \\
B & (u_{BA},v_{BA}) & (u_{BB},v_{BB})
\end{array}
\]
and code the extracted row-player action as Nash-supported if it belongs to the support of at least one Nash equilibrium of that matrix. Let \(q\) be the probability that the column player chooses \(A\), and let \(p\) be the probability that the row player chooses \(A\). The row player's payoff difference between choosing \(A\) and \(B\) is
\[
\Delta_u(q)=q(u_{AA}-u_{BA})+(1-q)(u_{AB}-u_{BB})
= (u_{AA}-u_{AB}-u_{BA}+u_{BB})q + (u_{AB}-u_{BB}),
\]
and the column player's payoff difference between choosing \(A\) and \(B\) is
\[
\Delta_v(p)=p(v_{AA}-v_{AB})+(1-p)(v_{BA}-v_{BB})
= (v_{AA}-v_{AB}-v_{BA}+v_{BB})p + (v_{BA}-v_{BB}).
\]
Pure equilibria are checked using weak best-response inequalities, for example \((A,A)\) is a Nash equilibrium if \(u_{AA}\ge u_{BA}\) and \(v_{AA}\ge v_{AB}\), with analogous inequalities for the other three cells. Mixed-support equilibria are checked by solving the relevant indifference equations \(\Delta_u(q)=0\) and/or \(\Delta_v(p)=0\) on the open unit interval, also allowing cases in which one player is indifferent while the other player's pure action is a best response. Formally, for observation \(t\), we define
\[
S_t=\{i\in\{A,B\}:\exists (\sigma_r,\sigma_c)\in NE(u_t,v_t)
\text{ such that } \sigma_r(i)>0\},
\qquad
I_t=\mathbf{1}\{a_t\in S_t\},
\]
where \(a_t\) is the extracted model action. The reported statistic for each model and prompting regime is \(\frac{1}{N}\sum_{t=1}^N I_t\). We use support membership rather than selecting a particular equilibrium because the model supplies only its own action; the opponent's realized action or mixed strategy is not observed.

\section{Description of Each Game}\label{app:GameDescription}
This section describes each game and the payoff-relevant parameters that vary
across decision instances. For each game \(g\), we draw these parameters
independently from simple uniform ranges and then elicit model choices for the
resulting instance. Textual environments change the narrative wrapper but not
the action space or these payoff-relevant quantities.

\paragraph{Dictator game.}
A single allocator chooses an integer transfer \(a \in \{0,\dots,E\}\) from an
endowment \(E\) to a passive recipient and keeps \(E-a\). We draw
\[
E \sim \mathrm{Unif}[250, 2000].
\]

\paragraph{Ultimatum game.}
A proposer chooses an integer offer \(a \in \{0,\dots,E\}\) from a stake \(E\).
The counterpart can accept or reject the proposed split; if accepted, the
counterpart receives \(a\) and the proposer keeps \(E-a\), while rejection gives
both parties zero. The model is asked for the proposer's offer. We draw
\[
E \sim \mathrm{Unif}[50{,}000, 200{,}000].
\]

\paragraph{Trust game.}
A first mover chooses an integer amount \(a \in \{0,\dots,E\}\) to send to a
partner. The sent amount is multiplied by \(k\) before reaching the partner, who
may then return some amount. The model is asked for the first mover's sent
amount. We draw
\[
E \sim \mathrm{Unif}[1{,}000, 10{,}000],
\qquad
k \sim \mathrm{Unif}[2, 4].
\]

\paragraph{Public-goods game.}
In a four-person group, the model chooses an integer contribution
\(a \in \{0,\dots,E\}\) to a common pool. Every dollar contributed by any group
member is multiplied by \(m\), and the resulting total is split equally among
the four participants. We draw
\[
E \sim \mathrm{Unif}[300, 2000],
\qquad
m \sim \mathrm{Unif}[1.1, 2.5].
\]

\paragraph{Beauty-contest game.}
Each participant chooses an integer \(a \in \{0,\dots,M\}\). The winning choice
is the number closest to \(p\) times the average of all submitted numbers, and
the winner receives a prize \(E\). The model is asked for its chosen number. We
draw
\[
M \sim \mathrm{Unif}[10, 50],
\qquad
p \sim \mathrm{Unif}[0.3, 1.8],
\qquad
E \sim \mathrm{Unif}[1{,}000, 2{,}000].
\]

\paragraph{Lottery-choice task.}
The model chooses between two binary lotteries, \(A\) and \(B\). Lottery
\(j\in\{A,B\}\) pays a high payoff \(H_j\) with probability \(p_j\) and a low
payoff \(L_j\) otherwise. We draw the two lotteries independently:
\[
p_A,p_B \sim \mathrm{Unif}[0.05, 0.95],
\qquad
H_A,H_B \sim \mathrm{Unif}[2{,}000, 10{,}000],
\qquad
L_A,L_B \sim \mathrm{Unif}[0, 1{,}000].
\]

\paragraph{Normal-form \(2\times2\) game.}
The model is the row player in a simultaneous-move \(2\times2\) game and chooses
between actions \(A\) and \(B\). The column player also chooses \(A\) or \(B\).
For each cell \((i,j)\in\{A,B\}^2\), \(u_{ij}\) is the row player's payoff and
\(v_{ij}\) is the column player's payoff, both shown in the prompt in thousands
of dollars. We draw all eight payoff entries independently:
\[
u_{ij}, v_{ij} \sim \mathrm{Unif}[0, 100],
\qquad i,j \in \{A,B\}.
\]

\newpage

\section{Additional Figures}
\label{app:additional-figures}

\begin{figure}[H]
    \centering
    \includegraphics[width=\textwidth]{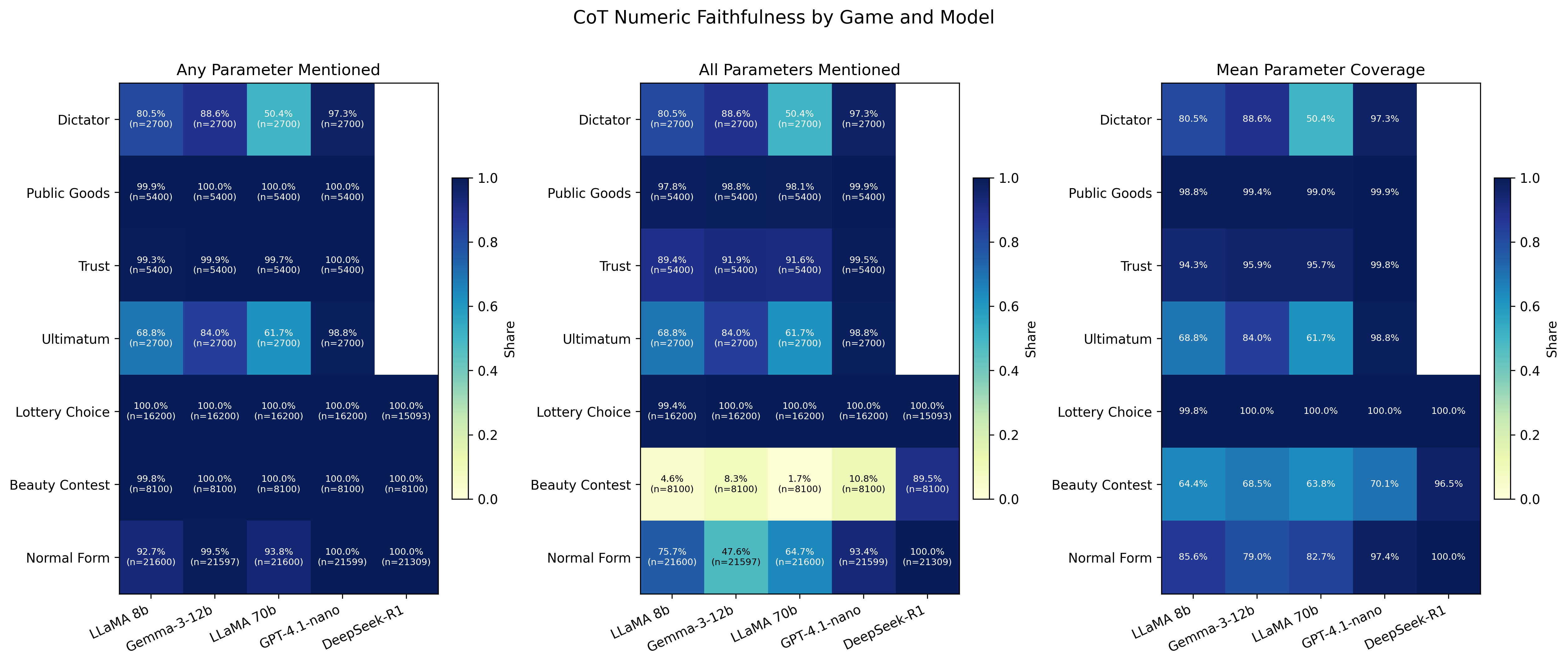}
    \caption{\textbf{Numeric grounding in chain-of-thought.}
    Share of CoT responses that mention payoff-relevant numerical parameters, by
    game and model. High coverage indicates that the reasoning text engages with
    the quantities defining the decision problem, while incomplete coverage leaves
    room for payoff-irrelevant framing to shape the representation.}
    \label{fig:numeric_faithfulness}
\end{figure}

\begin{figure}[H]
    \centering
    \includegraphics[width=\linewidth]{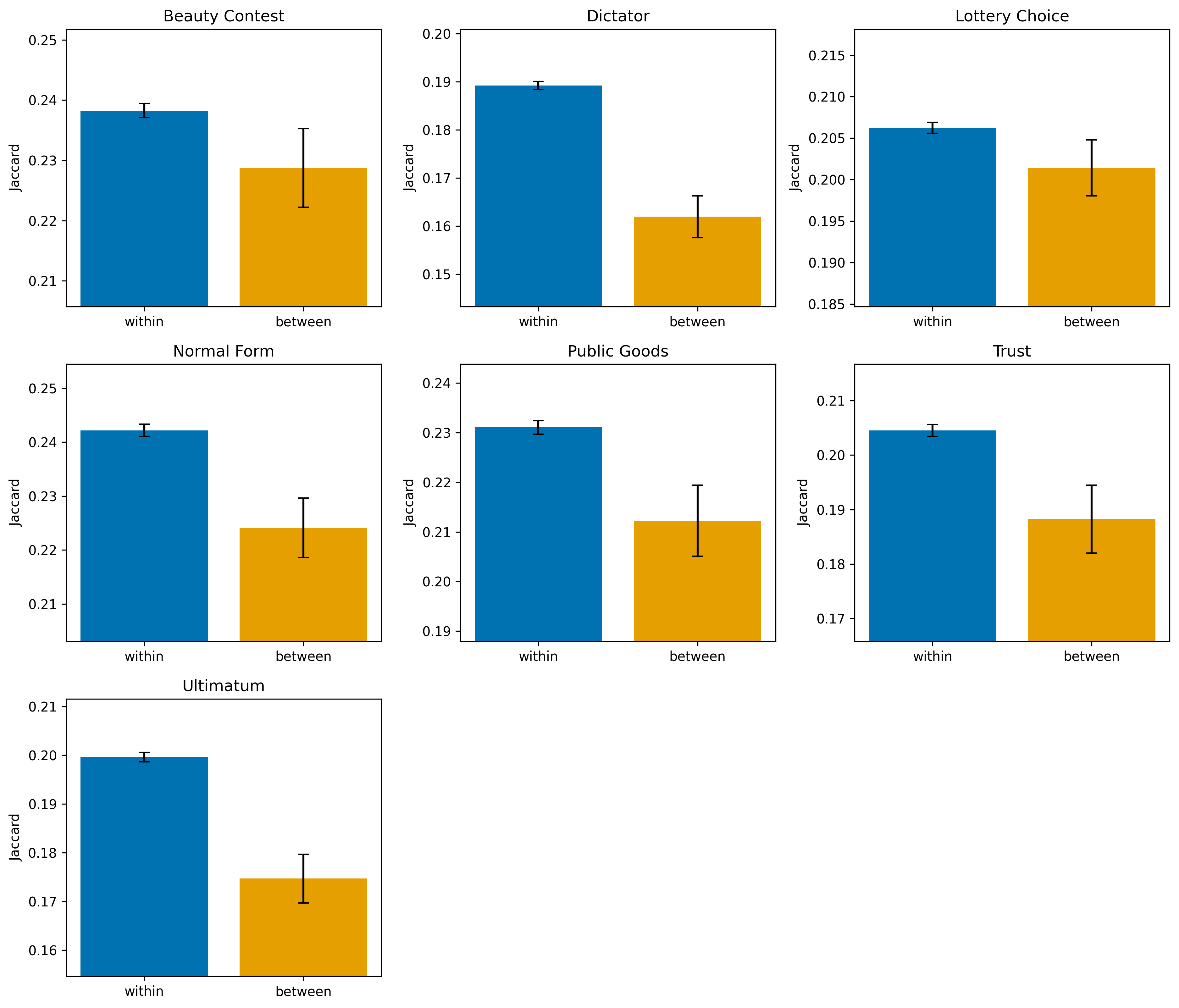}
    \caption{Mean Jaccard similarity (with standard error) of chain-of-thought token sets,
    comparing within-environment and between-environment pairs, by game.
    Higher within-environment similarity indicates greater environment-specific
    regularity in reasoning language despite payoff equivalence.
    }
    \label{fig:jaccard_cot}
\end{figure}

\begin{figure}[H]
    \centering
    \includegraphics[width=\textwidth]{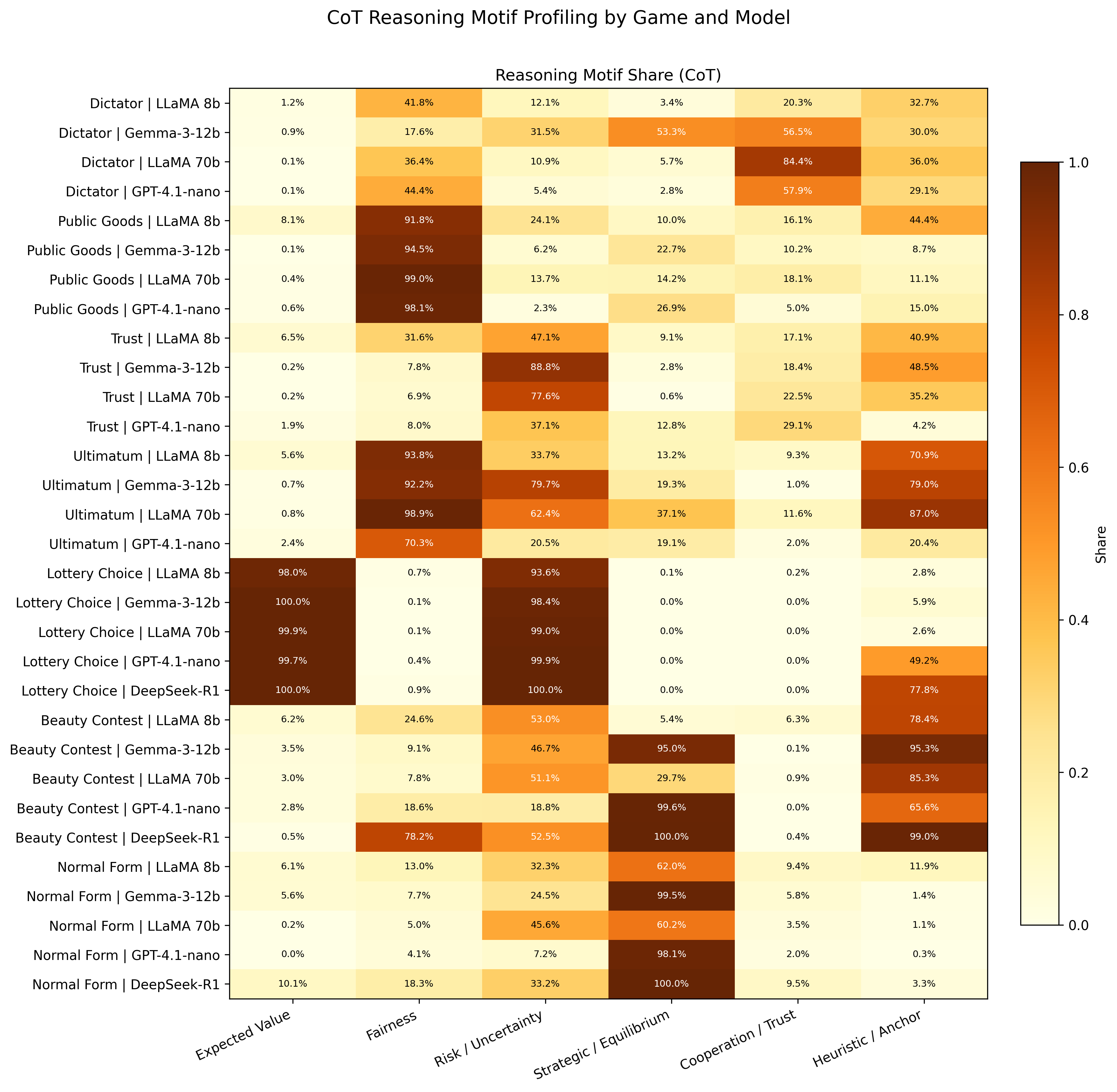}
    \caption{\textbf{Reasoning motifs by game and model.}
    Share of CoT responses containing each motif category. The motif profile
    reveals whether reasoning is organized around formal structure, such as
    expected value or strategic equilibrium, or around more frame-sensitive
    concepts, such as fairness, cooperation, and heuristics.}
    \label{fig:reasoning_motifs}
\end{figure}

\begin{figure}[H]
    \centering
    \includegraphics[width=\textwidth]{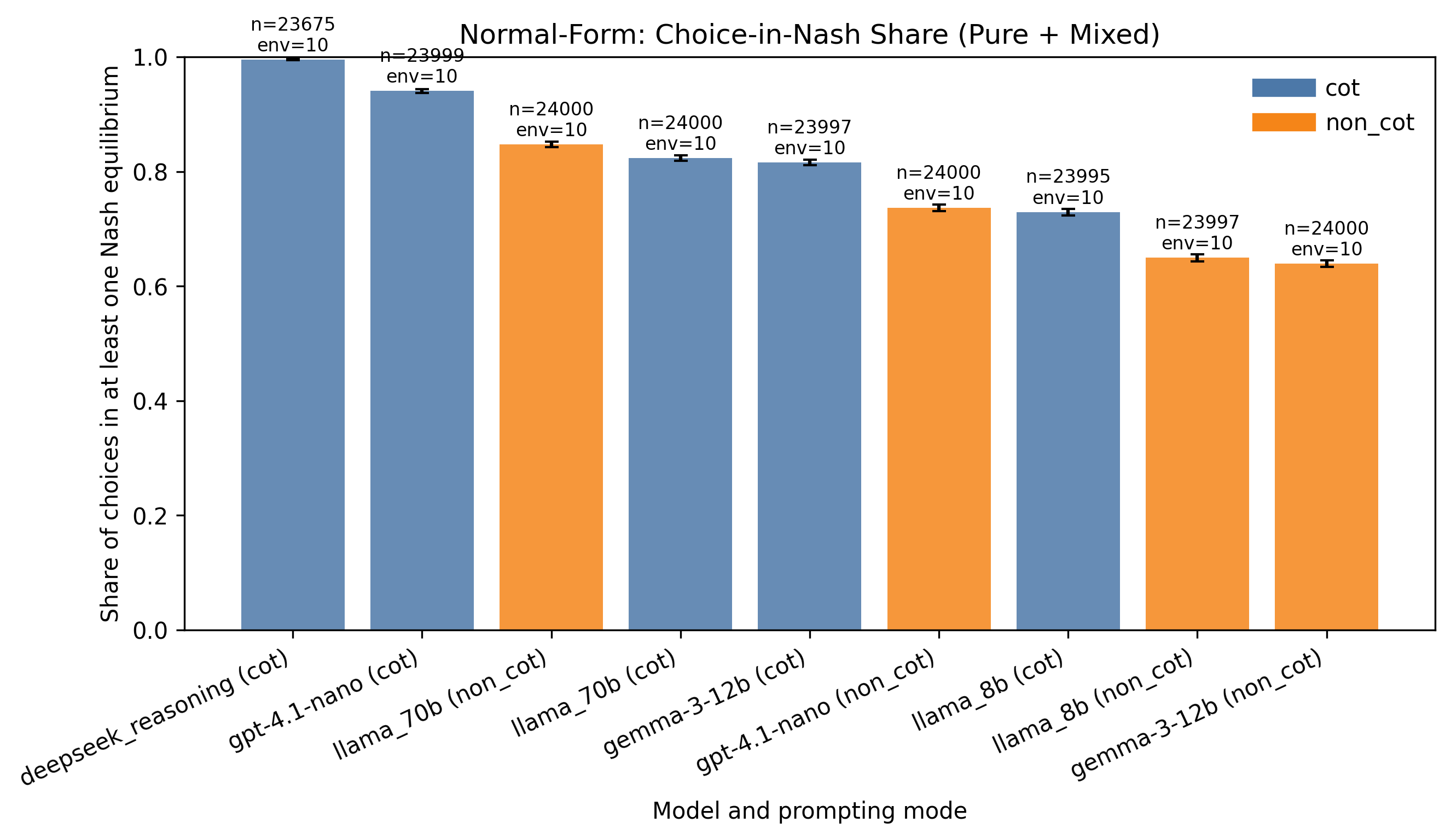}
    \caption{\textbf{Nash-supported choices in normal-form games.}
    Share of model choices that lie in the support of at least one Nash
    equilibrium, allowing pure and mixed equilibria. This diagnostic measures
    whether the model's action is compatible with an equilibrium representation
    of the underlying strategic problem.}
    \label{fig:nash_alignment}
\end{figure}

\begin{figure}[H]
    \centering
    \includegraphics[width=\textwidth]{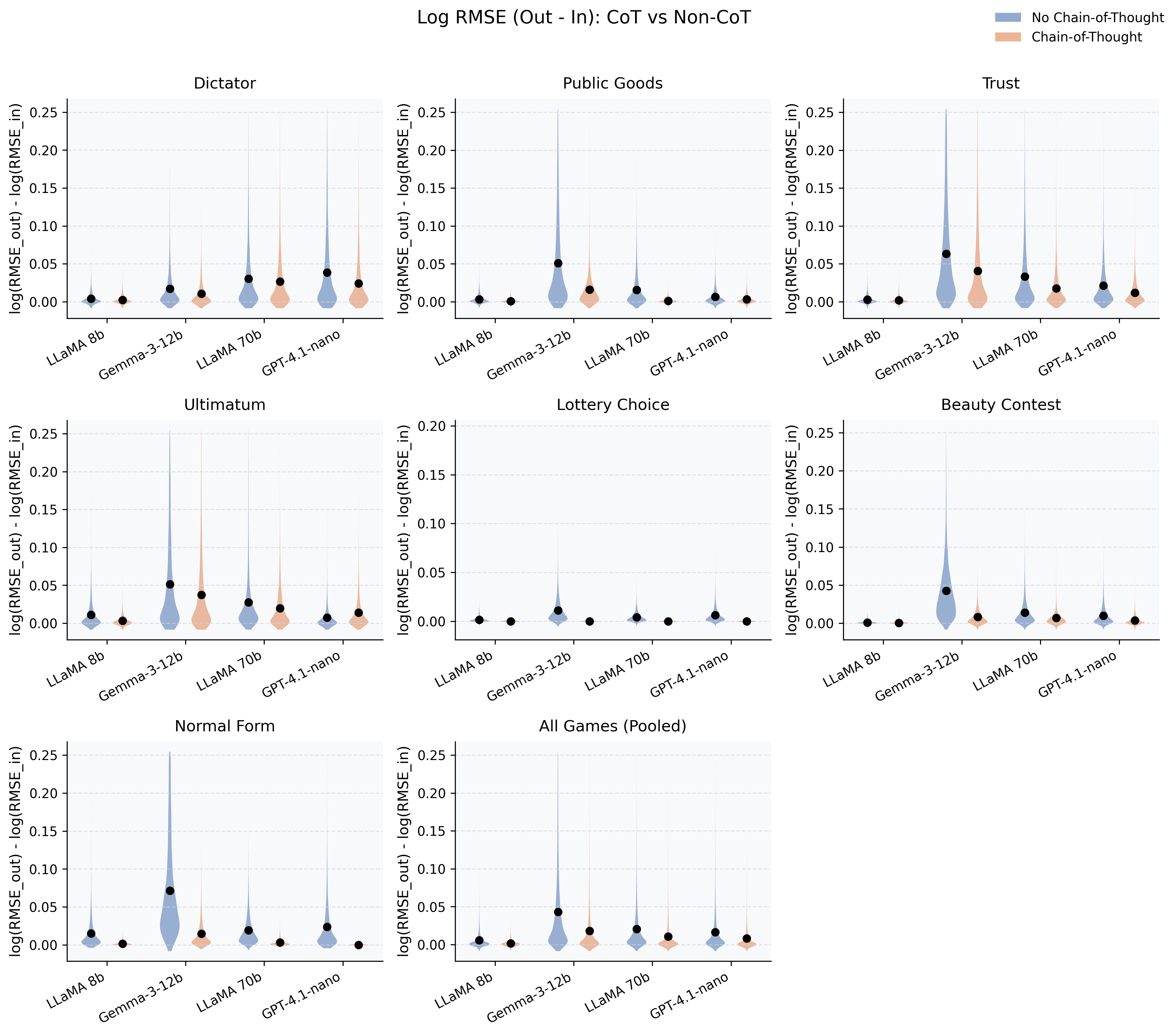}
    \caption{\textbf{Predictive portability gaps under prompt variation.}
    Distribution of $\log(\mathrm{RMSE}_{\text{out}})-\log(\mathrm{RMSE}_{\text{in}})$
    across random source--target environment resamplings, comparing non-CoT and
    CoT prompting. Values above zero indicate that a transported behavioral
    mapping predicts held-out target behavior worse than a target-trained
    benchmark.}
    \label{fig:rmse_cot_grid}
\end{figure}

\begin{figure}[H]
    \centering
    \includegraphics[width=\textwidth]{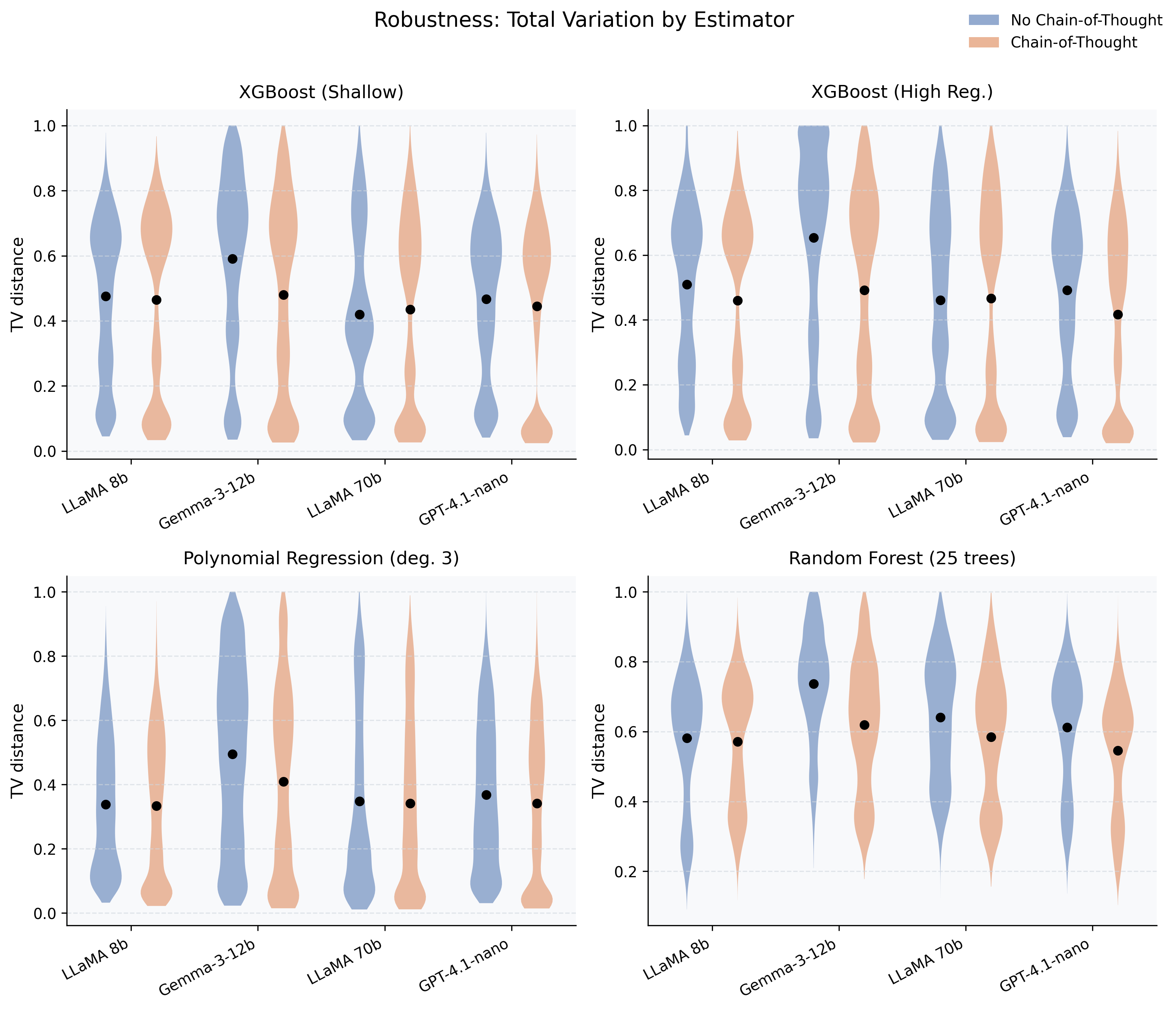}
    \caption{\textbf{TV robustness across behavioral predictors.}
    Total variation portability distributions estimated with nonlinear and
    regularized behavioral predictors.}
    \label{fig:robustness_tv_overview}
\end{figure}

\begin{figure}[H]
    \centering
    \includegraphics[width=\textwidth]{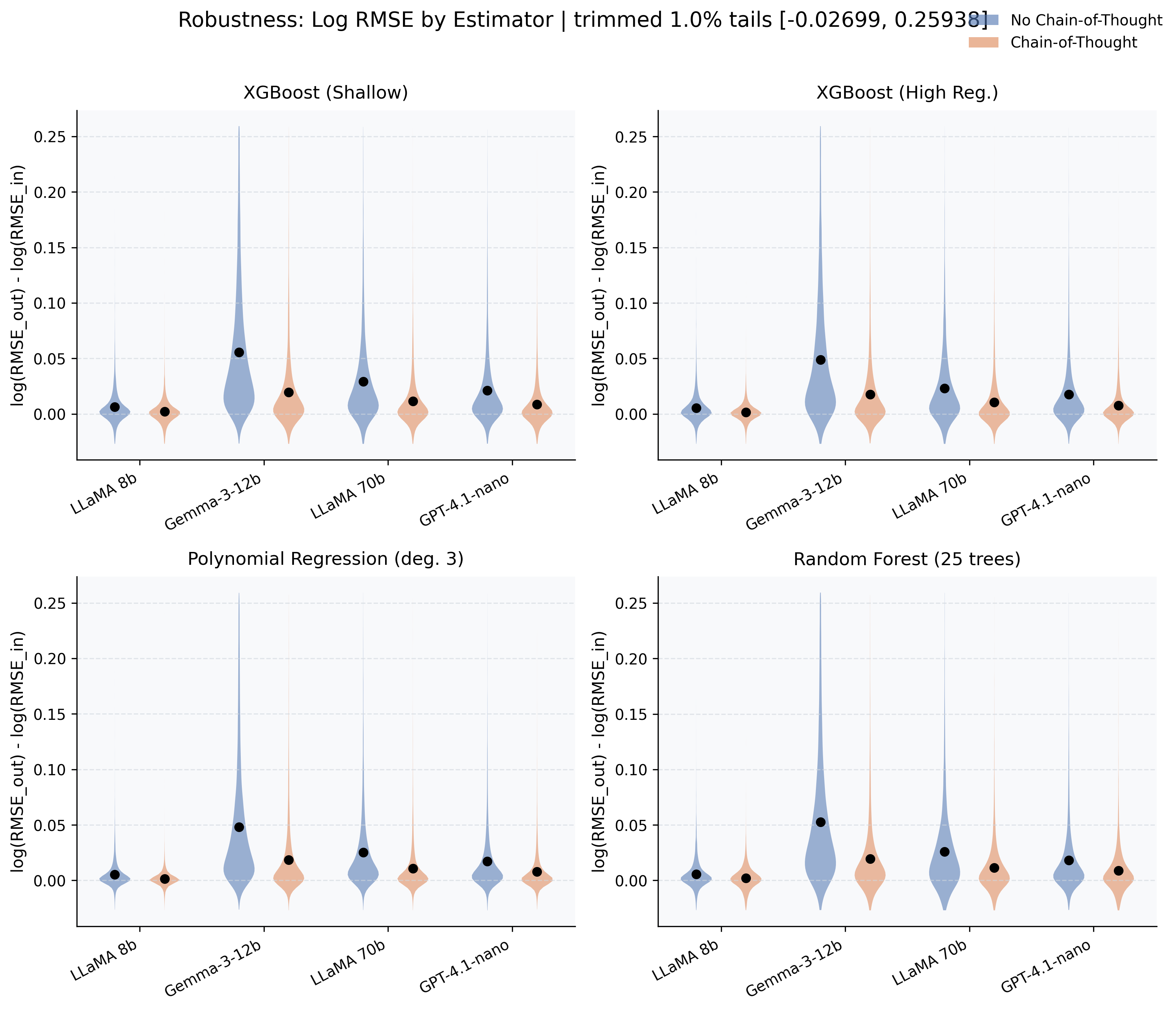}
    \caption{\textbf{Log-RMSE robustness across behavioral predictors.}
    Predictive portability gaps estimated with nonlinear and regularized
    behavioral predictors.}
    \label{fig:robustness_rmse_overview}
\end{figure}

\begin{figure}[H]
    \centering
    \begin{subfigure}[t]{0.48\textwidth}
        \centering
        \includegraphics[width=\linewidth]{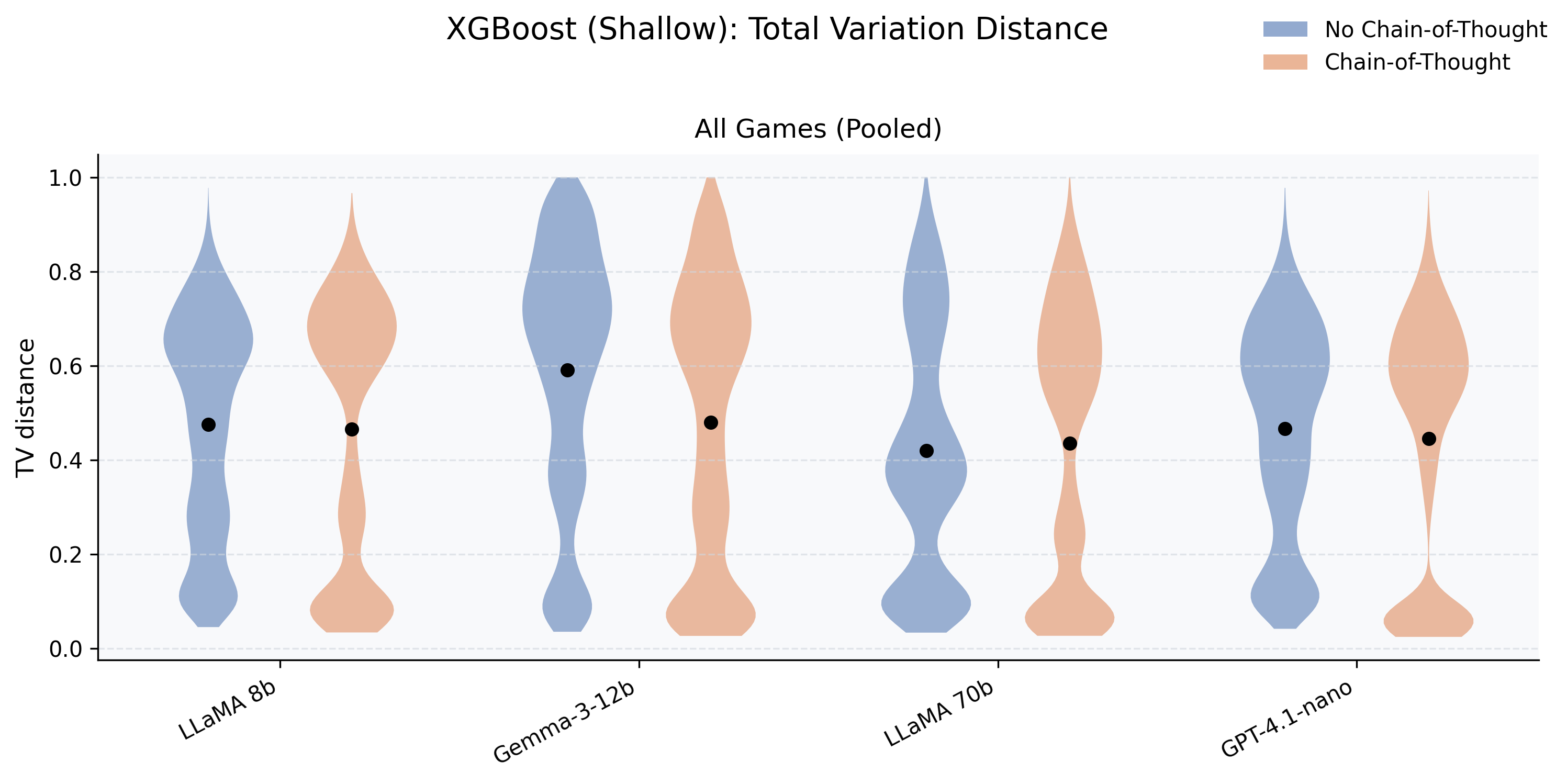}
        \caption{Shallow XGBoost}
    \end{subfigure}
    \hfill
    \begin{subfigure}[t]{0.48\textwidth}
        \centering
        \includegraphics[width=\linewidth]{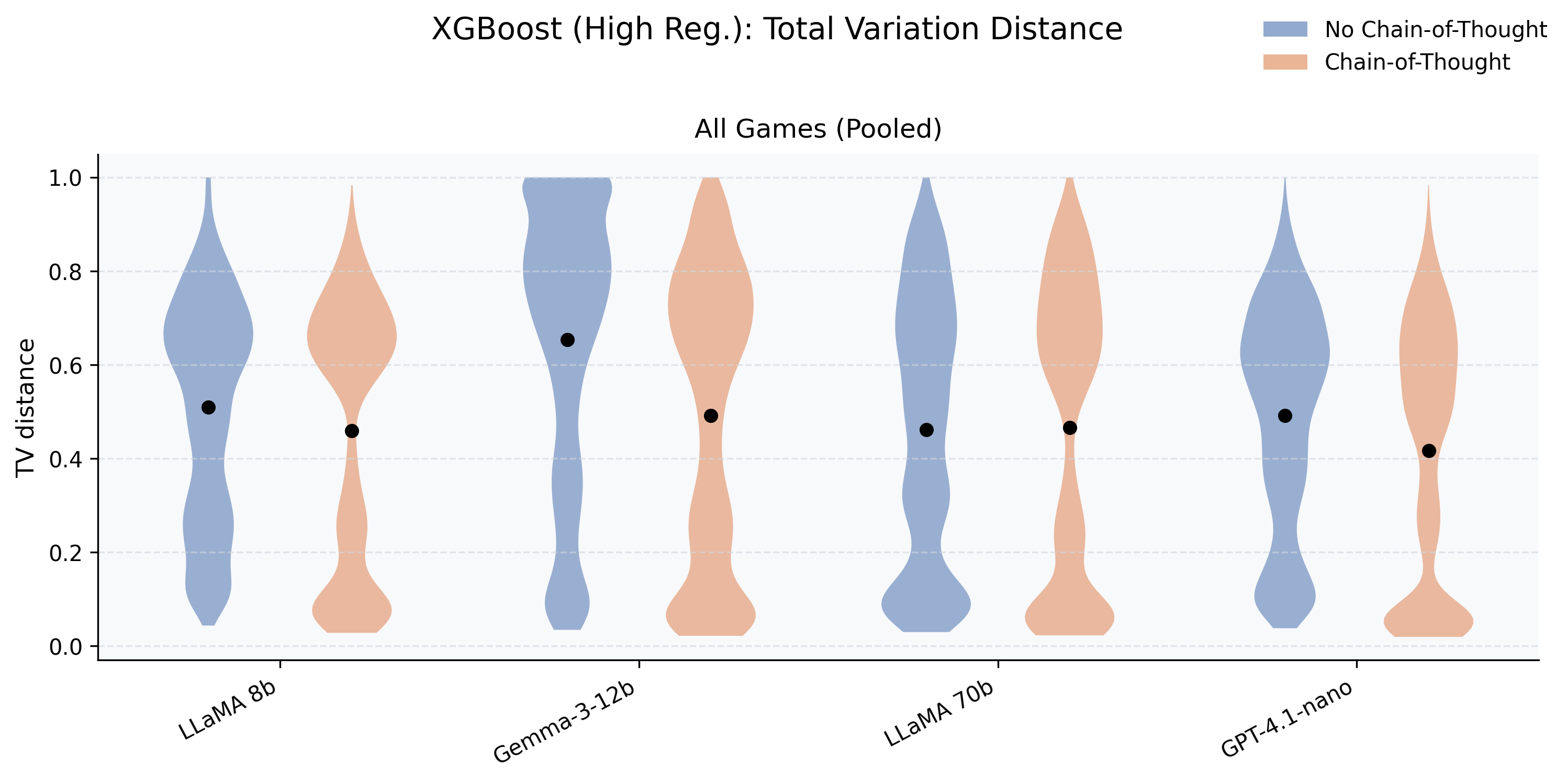}
        \caption{High-regularization XGBoost}
    \end{subfigure}
    \begin{subfigure}[t]{0.48\textwidth}
        \centering
        \includegraphics[width=\linewidth]{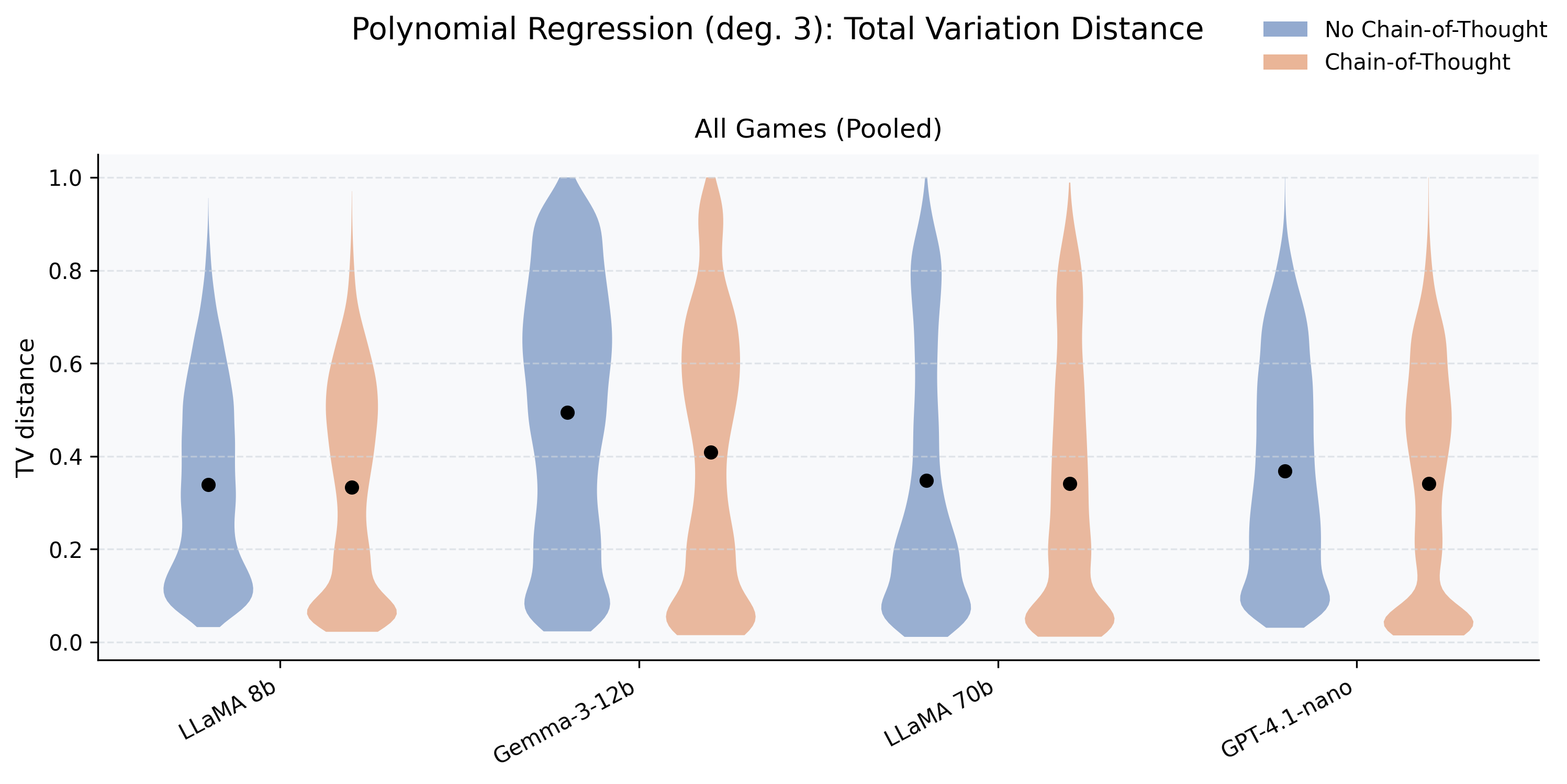}
        \caption{Degree-3 polynomial regression}
    \end{subfigure}
    \hfill
    \begin{subfigure}[t]{0.48\textwidth}
        \centering
        \includegraphics[width=\linewidth]{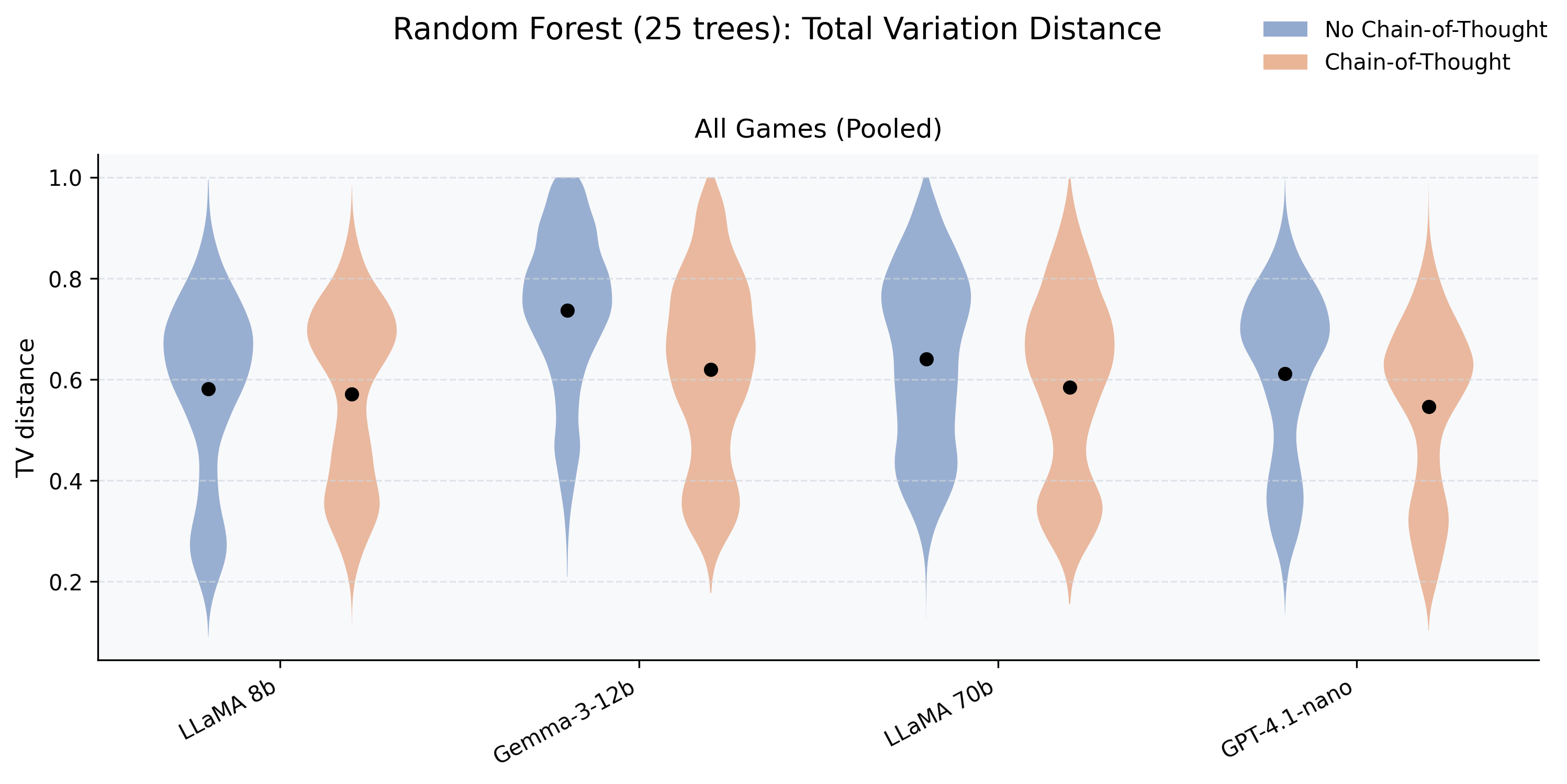}
        \caption{Random forest}
    \end{subfigure}
    \caption{\textbf{Pooled TV robustness by estimator.}
    Pooled TV distributions for each nonlinear or regularized behavioral
    predictor.}
    \label{fig:robustness_tv_pooled}
\end{figure}

\begin{figure}[H]
    \centering
    \begin{subfigure}[t]{0.48\textwidth}
        \centering
        \includegraphics[width=\linewidth]{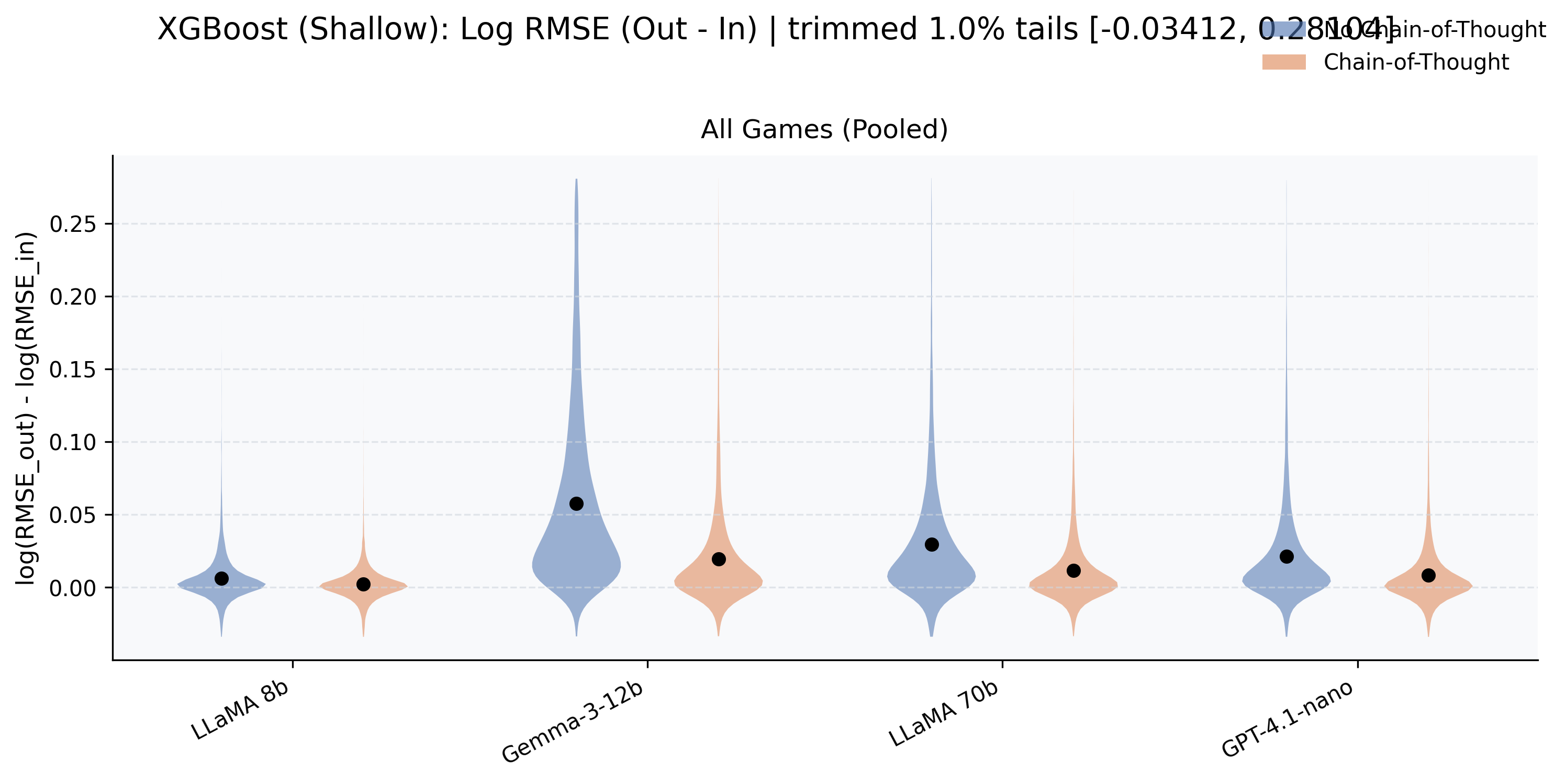}
        \caption{Shallow XGBoost}
    \end{subfigure}
    \hfill
    \begin{subfigure}[t]{0.48\textwidth}
        \centering
        \includegraphics[width=\linewidth]{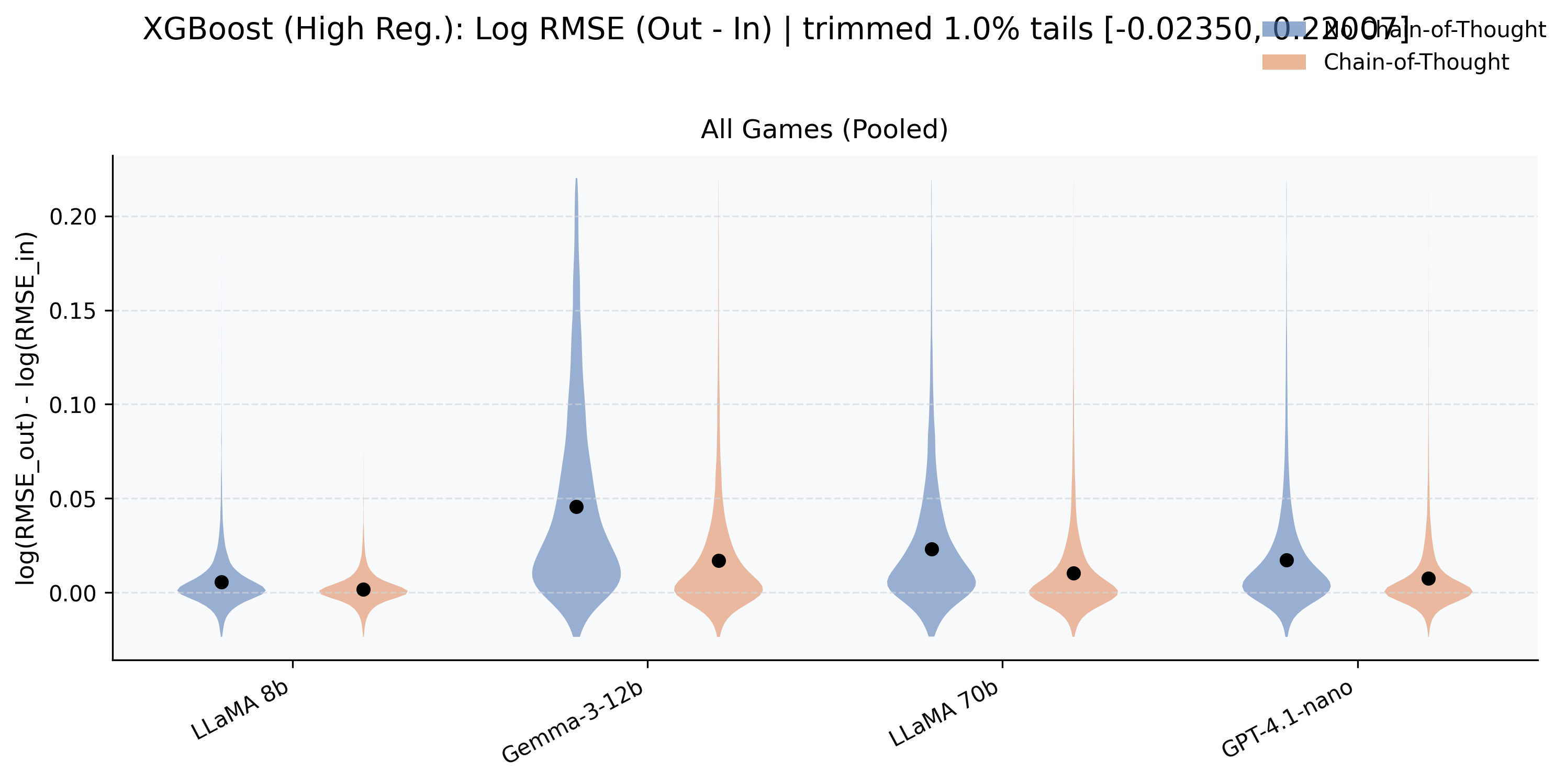}
        \caption{High-regularization XGBoost}
    \end{subfigure}
    \begin{subfigure}[t]{0.48\textwidth}
        \centering
        \includegraphics[width=\linewidth]{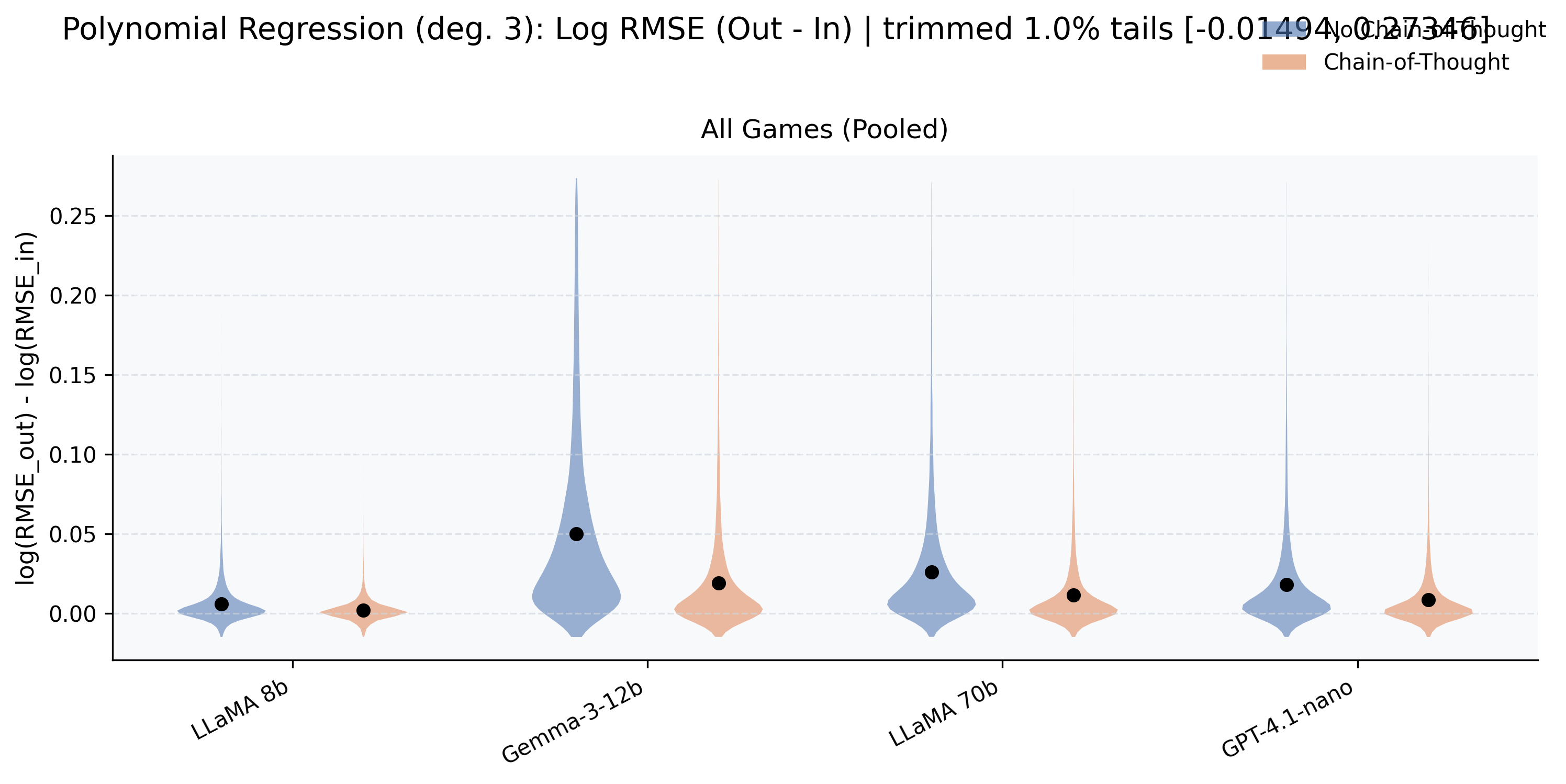}
        \caption{Degree-3 polynomial regression}
    \end{subfigure}
    \hfill
    \begin{subfigure}[t]{0.48\textwidth}
        \centering
        \includegraphics[width=\linewidth]{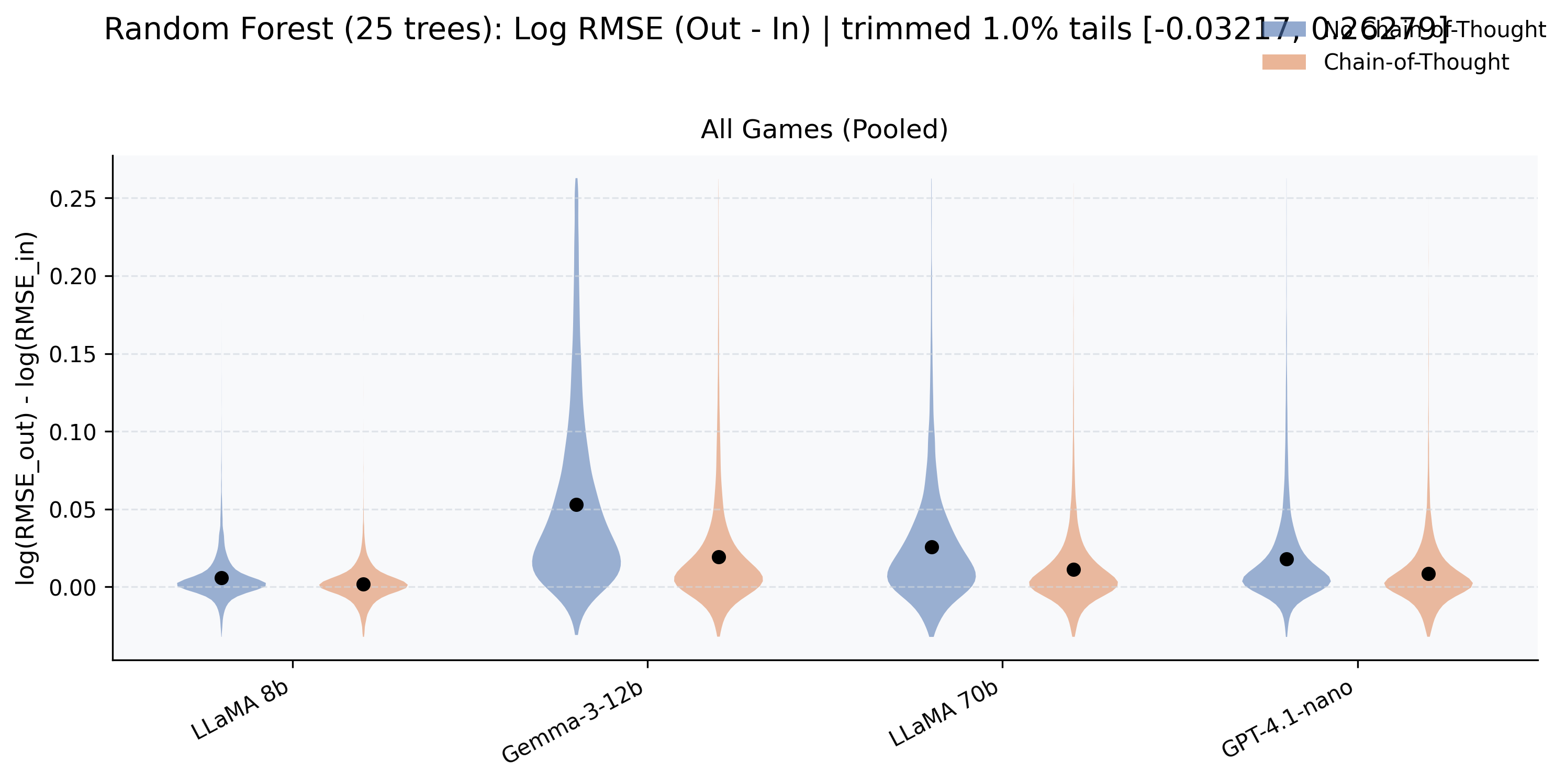}
        \caption{Random forest}
    \end{subfigure}
    \caption{\textbf{Pooled log-RMSE robustness by estimator.}
    Pooled predictive portability gaps for each nonlinear or regularized
    behavioral predictor.}
    \label{fig:robustness_rmse_pooled}
\end{figure}

\begin{figure}[H]
    \centering
    \includegraphics[width=\textwidth]{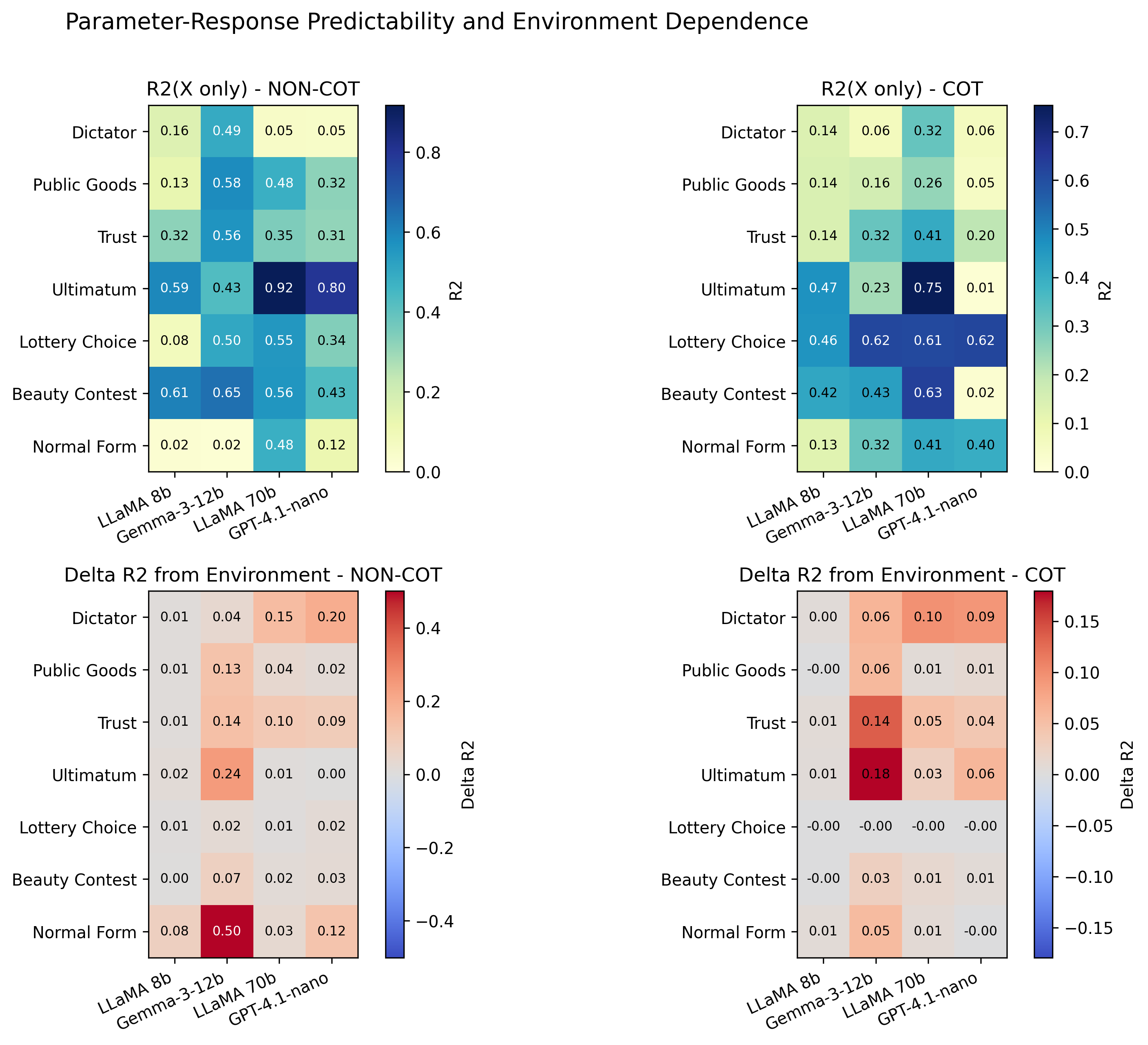}
    \caption{\textbf{Predictability from payoff variables and textual environment.}
    Comparison of predictive \(R^2\) using payoff variables alone versus payoff
    variables plus environment indicators. Larger incremental \(R^2\) from the
    environment indicates greater residual frame dependence.}
    \label{fig:parameter_predictability}
\end{figure}

\FloatBarrier

\section{Prompts and Environment Variations}
\label{app:prompts}

This appendix documents the prompts used to instantiate each game and the
payoff-irrelevant environment variations. Each query to the LLM consists of a
\emph{system} message and a \emph{user} message. Payoff-relevant parameters
(e.g., endowment $E$, multiplier $m$, probability $p$) are inserted via
string-formatting placeholders such as \(\{E:.0f\}\).

For each game we define
a \emph{base} environment and $40$ alternative environments
(\emph{env01}--\emph{env40}) that change only narrative and surface form
(i.e., payoff-irrelevant features). These
alternatives were generated with ChatGPT as narrative rewrites of the base case, with 2 human-written environments provided. Each environment changes only payoff-irrelevant language (scenario framing, domain
metaphor, descriptive ``fluff,'' and minor formatting) while preserving the
action space and payoff-relevant quantities. Where the game involves another
party or group, the prompt family preserves the strategic role and payoff
mapping of that party while allowing the narrative domain to vary as part of
\(Z(e)\). For example, the Dictator environments keep the recipient as
``Recipient \#27,'' the Ultimatum environments keep the counterpart as
``a collaborator,'' and the Normal Form environments keep the same
row-player action set and payoff matrix while changing the applied domain.

\paragraph{Prompting regimes.}
We evaluate two prompting regimes. In the \textbf{non-CoT} regime, we instruct
the model to respond with \emph{only} the final action. In the \textbf{CoT}
regime, we append an instruction to provide step-by-step reasoning, followed by
a final line in a fixed format.

\begin{quote}\small
\textbf{Non-CoT suffix (appended to the user prompt):}\\
\texttt{Reply with ONLY the final answer for this decision: just the number or letter, with no explanation or extra text.}

\vspace{0.5em}
\textbf{CoT suffix (appended to the user prompt):}\\
\texttt{First, think through the problem step by step and explain your reasoning.}\\
\texttt{Then, on the last line, write only:}\\
\texttt{FINAL\_ANSWER: <your answer here>}
\end{quote}

\subsection{Game Prompt Templates}
\label{app:prompt-templates}

We provide the exact \texttt{system} and \texttt{user} templates for the base
environment of each game. For brevity, we show representative environment
variants in Appendix~\ref{app:env-examples}. Extraction prompts used to parse
actions from raw outputs are listed in Appendix~\ref{app:extraction}.

\subsubsection{Dictator Game (base)}
\label{app:dictator-base}
\begin{quote}\small
\textbf{System:}\\
\texttt{Standard allocation experiment.}

\textbf{User:}\\
\texttt{You have an endowment of \{E:.0f\} dollars. You may transfer any portion to
**another anonymous participant in the experiment** and keep the rest.}\\
\texttt{Reply with a single integer between 0 and \{E:.0f\}: the number of dollars you choose to give away.}
\end{quote}

\subsubsection{Ultimatum Game (base)}
\label{app:ultimatum-base}
\begin{quote}\small
\textbf{System:}\\
\texttt{Standard ultimatum experiment.}

\textbf{User:}\\
\texttt{You have an endowment of \{E:.0f\} dollars. Decide how many dollars to offer to
**another anonymous participant in the experiment**. If they accept, they receive the offer and you keep the remainder.
If they reject, you both receive 0.}\\
\texttt{Reply with a single integer between 0 and \{E:.0f\}: the amount you offer.}
\end{quote}

\subsubsection{Trust Game (base)}
\label{app:trust-base}
\begin{quote}\small
\textbf{System:}\\
\texttt{Standard trust-game experiment.}

\textbf{User:}\\
\texttt{You hold \{E:.0f\} dollars. Any dollars you send to **another anonymous participant in the experiment**
will be multiplied by \{k:.0f\} before they receive it. They may then choose how much, if any, to return to you.}\\
\texttt{Reply with a single integer between 0 and \{E:.0f\}: the amount you send.}
\end{quote}

\subsubsection{Public Goods Game (base)}
\label{app:pg-base}
\begin{quote}\small
\textbf{System:}\\
\texttt{Standard public-goods experiment with four participants.}

\textbf{User:}\\
\texttt{You have \{E:.0f\} dollars. You may contribute any portion to a common pool. Every dollar contributed by any participant will be multiplied by \{m:.1f\}, and the resulting total will be split equally among all four participants.}\\
\texttt{Reply with a single integer between 0 and \{E:.0f\}: the number of dollars you choose to contribute.}
\end{quote}

\subsubsection{Beauty Contest (base)}
\label{app:bc-base}
\begin{quote}\small
\textbf{System:}\\
\texttt{Standard beauty-contest experiment.}

\textbf{User:}\\
\texttt{You and several other participants each choose an integer between 0 and \{M:.0f\}. The winner is the person whose number is closest to \{p:.2f\} times the group's average.}\\
\texttt{The winner gets a prize of \{E:.0f\} dollars. Reply with one integer 0--\{M:.0f\}: your chosen number.}
\end{quote}

\subsubsection{Lottery Choice (base)}
\label{app:lc-base}
\begin{quote}\small
\textbf{System:}\\
\texttt{Standard lottery-choice experiment.}

\textbf{User:}\\
\texttt{You must choose between two lotteries:}\\
\texttt{**Option A** -- \{pA:.0\%\} chance of winning \$\{HA:.0f\}, otherwise \$\{LA:.0f\}.}\\
\texttt{**Option B** -- \{pB:.0\%\} chance of winning \$\{HB:.0f\}, otherwise \$\{LB:.0f\}.}\\
\texttt{Reply with one letter (uppercase): **A** or **B**: which option would you choose.}
\end{quote}

\subsubsection{Normal-Form 2-by-2 Game (base)}
\label{app:nf-base}
\begin{quote}\small
\textbf{System:}\\
\texttt{Standard 2-by-2 normal-form game.}

\textbf{User:}\\
\texttt{You and another participant simultaneously choose between **Strategy A** and **Strategy B**.}\\
\texttt{Payoff matrix (your payoff shown first, the other player's second, all in thousands of dollars):}\\
\texttt{If both choose **A**: (\{uAA:.0f\}, \{vAA:.0f\})}\\
\texttt{If you choose **A** and they choose **B**: (\{uAB:.0f\}, \{vAB:.0f\})}\\
\texttt{If you choose **B** and they choose **A**: (\{uBA:.0f\}, \{vBA:.0f\})}\\
\texttt{If both choose **B**: (\{uBB:.0f\}, \{vBB:.0f\})}\\
\texttt{Reply with one letter (uppercase): **A** or **B**: which action would you choose.}
\end{quote}

\subsection{Environment Variations}
\label{app:env-examples}

\paragraph{Dictator (example environments).}
\begin{quote}\small
\textbf{env01 System:} \texttt{Scenario: Sharing a surprise performance bonus.}\\
\textbf{env01 User:} \texttt{You receive a modest surprise performance bonus of \{E:.0f\} dollars. You may keep any or all of this money, or transfer any portion of it to **an anonymous Recipient \#27 that was randomly assigned to you**. Reply with a single integer between 0 and \{E:.0f\}: the number of dollars you choose to transfer to Recipient \#27.}

\vspace{0.5em}
\textbf{env02 System:} \texttt{Scenario: Distributing leftover project funds.}\\
\textbf{env02 User:} \texttt{Your research project finishes under budget, leaving \{E:.0f\} dollars in surplus funds that you may personally use. You may keep any or all of this money, or transfer any portion of it to **an anonymous Recipient \#27 that was randomly assigned to you**. Reply with a single integer between 0 and \{E:.0f\}: the number of dollars you choose to transfer to Recipient \#27.}
\end{quote}

\paragraph{Ultimatum (example environments).}
\begin{quote}\small
\textbf{env01 System:} \texttt{Scenario: Splitting a research grant.}\\
\textbf{env01 User:} \texttt{You receive a large discretionary research grant of \{E:.0f\} dollars to share with a collaborator. You must decide how much to offer to **a collaborator who was also involved in the work**. They can accept or reject your proposed split. If they accept, they receive the amount you offer and you keep the remainder. If they reject, you both receive 0. Reply with a single integer between 0 and \{E:.0f\}: the amount you choose to offer.}

\vspace{0.5em}
\textbf{env02 System:} \texttt{Scenario: Dividing book royalties.}\\
\textbf{env02 User:} \texttt{Your publisher pays a substantial advance of \{E:.0f\} dollars for a co-authored book. You must decide how much to offer to **a collaborator who was also involved in the work**. They can accept or reject your proposed split. If they accept, they receive the amount you offer and you keep the remainder. If they reject, you both receive 0. Reply with a single integer between 0 and \{E:.0f\}: the amount you choose to offer.}
\end{quote}

\paragraph{Trust (example environments).}
\begin{quote}\small
\textbf{env01 System:} \texttt{Scenario: Funding a joint marketing campaign.}\\
\textbf{env01 User:} \texttt{You and a business partner run a small company. You have \{E:.0f\} dollars available in the business account, and your partner proposes using some of it for a joint marketing campaign. You may keep any or all of it, or send some portion to **a business partner**. Whatever you send will be multiplied by \{k:.0f\} before they receive it. After that, they may choose how much, if any, to return to you. Reply with a single integer between 0 and \{E:.0f\}: the amount you choose to send to them.}

\vspace{0.5em}
\textbf{env02 System:} \texttt{Scenario: Opening a new branch location.}\\
\textbf{env02 User:} \texttt{You share ownership of a business with a partner. You have \{E:.0f\} dollars set aside, and your partner suggests using some of it to open a new branch. You may keep any or all of it, or send some portion to **a business partner**. Whatever you send will be multiplied by \{k:.0f\} before they receive it. After that, they may choose how much, if any, to return to you. Reply with a single integer between 0 and \{E:.0f\}: the amount you choose to send to them.}
\end{quote}

\paragraph{Public Goods (example environments).}
\begin{quote}\small
\textbf{env01 System:} \texttt{Scenario: Funding shared lab equipment.}\\
\textbf{env01 User:} \texttt{Your four-person lab has individual budgets of \{E:.0f\} dollars each to spend on research. You may keep any amount for yourself or contribute any portion of it to **a shared equipment fund**. Every dollar anyone contributes is multiplied by \{m:.1f\}, and the resulting total is split equally among all four members. Reply with a single integer between 0 and \{E:.0f\}: the number of dollars you choose to contribute.}

\vspace{0.5em}
\textbf{env02 System:} \texttt{Scenario: Contributing to a team training budget.}\\
\textbf{env02 User:} \texttt{Your four-person team each receives \{E:.0f\} dollars that can be used for personal or shared training. You may keep any amount for yourself or contribute any portion of it to **a pooled training budget**. Every dollar anyone contributes is multiplied by \{m:.1f\}, and the resulting total is split equally among all four members. Reply with a single integer between 0 and \{E:.0f\}: the number of dollars you choose to contribute.}
\end{quote}

\paragraph{Beauty Contest (example environments).}
\begin{quote}\small
\textbf{env01 System:} \texttt{Scenario: Forecasting a stock index sentiment score.}\\
\textbf{env01 User:} \texttt{You and several other participants each choose an integer between 0 and \{M:.0f\}. Your number is meant to represent the percentage score representing bullish sentiment on a stock index. However, the winner is not the one whose number is closest to the actual percentage, but the one whose number is closest to \{p:.2f\} times the average of all chosen numbers. The winner gets a prize of \{E:.0f\} dollars. Reply with one integer between 0 and \{M:.0f\}: your chosen number.}

\vspace{0.5em}
\textbf{env02 System:} \texttt{Scenario: Predicting streaming-platform approval.}\\
\textbf{env02 User:} \texttt{You and several other participants each choose an integer between 0 and \{M:.0f\}. Your number is meant to represent the percentage of early users who will give a new streaming show a five-star rating. However, the winner is not the one whose number is closest to the actual percentage, but the one whose number is closest to \{p:.2f\} times the average of all chosen numbers. The winner gets a prize of \{E:.0f\} dollars. Reply with one integer between 0 and \{M:.0f\}: your chosen number.}
\end{quote}

\paragraph{Lottery Choice (example environments).}
\begin{quote}\small
\textbf{env01 System:} \texttt{Scenario: Choosing between two travel insurance policies.}\\
\textbf{env01 User:} \texttt{You are buying travel insurance for an expensive international trip. The insurer offers you a choice between two payout structures if something goes wrong. **Option A** -- With probability \{pA:.0\%\}, you receive \$\{HA:.0f\}; otherwise you receive \$\{LA:.0f\}. **Option B** -- With probability \{pB:.0\%\}, you receive \$\{HB:.0f\}; otherwise you receive \$\{LB:.0f\}. Reply with one letter (uppercase): **A** or **B** -- which option would you choose?}

\vspace{0.5em}
\textbf{env02 System:} \texttt{Scenario: Selecting an annual bonus scheme at work.}\\
\textbf{env02 User:} \texttt{Your employer lets you choose between two possible bonus schemes for the coming year. **Option A** -- With probability \{pA:.0\%\}, you receive \$\{HA:.0f\}; otherwise you receive \$\{LA:.0f\}. **Option B** -- With probability \{pB:.0\%\}, you receive \$\{HB:.0f\}; otherwise you receive \$\{LB:.0f\}. Reply with one letter (uppercase): **A** or **B** -- which option would you choose?}
\end{quote}

\paragraph{Normal Form (example environments).}
\begin{quote}\small
\textbf{env01 System:} \texttt{Scenario: Pricing rivalry between two competing cafes.}\\
\textbf{env01 User:} \texttt{You and a rival cafe on the same street must simultaneously choose the price level for the next month. Your joint choices determine each cafe's profit from that month. Strategy A corresponds to keeping prices HIGH; Strategy B corresponds to cutting prices LOW. Both decisions are made simultaneously, and neither of you can see the other player's choice in advance. Payoff matrix (your payoff shown first, the other player's second, all in thousands of dollars): If both choose **A**: (\{uAA:.0f\}, \{vAA:.0f\}); if you choose **A** and they choose **B**: (\{uAB:.0f\}, \{vAB:.0f\}); if you choose **B** and they choose **A**: (\{uBA:.0f\}, \{vBA:.0f\}); if both choose **B**: (\{uBB:.0f\}, \{vBB:.0f\}). Reply with one letter (uppercase): **A** or **B** -- which action would you choose?}

\vspace{0.5em}
\textbf{env02 System:} \texttt{Scenario: Capacity planning for a shared regional airline route.}\\
\textbf{env02 User:} \texttt{You and a competing airline both serve the same regional route. You must simultaneously choose a capacity plan for the next travel season. Your joint choices determine each airline's profit from that route over the season. Strategy A is a HIGH-capacity plan; Strategy B is a LOW-capacity plan. Both decisions are made simultaneously, and neither of you can see the other player's choice in advance. Payoff matrix (your payoff shown first, the other player's second, all in thousands of dollars): If both choose **A**: (\{uAA:.0f\}, \{vAA:.0f\}); if you choose **A** and they choose **B**: (\{uAB:.0f\}, \{vAB:.0f\}); if you choose **B** and they choose **A**: (\{uBA:.0f\}, \{vBA:.0f\}); if both choose **B**: (\{uBB:.0f\}, \{vBB:.0f\}). Reply with one letter (uppercase): **A** or **B** -- which action would you choose?}
\end{quote}

\subsection{Extraction Prompts, Parsing and Human Validation}
\label{app:extraction}

All raw textual completions are post-processed by a separate extractor model (GPT-4.1-mini) that maps free-form responses into a single action. The extractor is queried with a fixed system message and a game-specific extraction prompt that is instantiated with the numerical parameters $x$ of the current decision instance.

For each game and textual environment, we define an extraction template which
is a short natural-language instruction describing the valid action space
(e.g.\ an integer range \([0,E]\) or a categorical choice between \(A\) and
\(B\)) and asking the extractor to recover the final choice. The extraction
templates used in the implementation are:
\begin{quote}\small
\textbf{Dictator:} Extract the final integer amount
\((0\text{--}\{E:.0f\})\) the participant chose to give away or transfer.
Return \texttt{None} if no clear integer answer is provided.

\textbf{Ultimatum:} Extract the final integer amount
\((0\text{ to }\{E:.0f\})\) the participant chose to offer. Return
\texttt{None} if no clear integer answer is provided.

\textbf{Trust:} Extract the final integer amount
\((0\text{ to }\{E:.0f\})\) the participant chose to send. Return
\texttt{None} if no clear integer answer is provided.

\textbf{Public Goods:} Extract the final integer amount
\((0\text{--}\{E:.0f\})\) contributed. Return \texttt{None} if unclear.

\textbf{Beauty Contest:} Extract the final integer
\((0\text{--}\{M:.0f\})\) the participant chose. Return \texttt{None} if no
clear integer answer is provided.

\textbf{Lottery Choice:} Extract the final choice (\texttt{A} or \texttt{B}).
Return \texttt{None} if no clear choice is provided.

\textbf{Normal Form:} Extract the final choice (\texttt{A} or \texttt{B}).
Return \texttt{None} if no clear choice is provided.
\end{quote}

The extractor call uses a fixed system message
\begin{quote}You are a helpful assistant that extracts the final answer from an open-ended response. If there is no clear answer, return None.
\end{quote}

We apply the same extraction prompt template across all environments of a given game: the extraction instructions depend only on the payoff-relevant parameters $x$ and the action space $A$, not on the narrative framing or surface wording $Z(e)$.

\paragraph{Human Validation.} We manually checked a stratified sample of 280 parses, balanced across game/model/CoT strata. All sampled extracted outputs are correct relative to the raw response. We also checked rows with a missing extracted value, and the rate is 196 / 1,901,107 = 0.0103\%. This remains small across conditions: 0.0107\% for CoT and 0.0022\% for non-CoT. At a specific game/model/environment/CoT condition, the largest observed extraction failure rate is for the public-goods / Llama-8B / env09 / CoT cell, at 0.826\%. These validation results suggest that extraction and parsing failures are too rare to plausibly account for the main portability findings.

\end{document}